\documentclass{article}
\usepackage[utf8]{inputenc}
\usepackage{graphicx}
\usepackage{amsmath,amsfonts,amssymb}
\usepackage{amsthm}
\usepackage[super]{nth}
\usepackage{multirow}
\usepackage{graphicx}
\usepackage{hyperref}

\usepackage{enumitem}
\usepackage[strict]{changepage}
\usepackage{comment} 
\usepackage{bbm}
\usepackage{tikz}
\usepackage{ctable}
\theoremstyle{definition}
\newtheorem{definition}{Definition}[section]
\newtheorem{theorem}{Theorem}[section]
\newtheorem{lemma}[theorem]{Lemma}
\newtheorem{conjecture}[theorem]{Conjecture}
\newtheorem{corollary}[theorem]{Corollary}
\newtheorem{proposition}[theorem]{Proposition}
\newtheorem{assumption}[theorem]{Assumption}

\newcommand{\HRule}{\rule{\linewidth}{0.2mm}}

\usepackage{hyperref}
\hypersetup{
    colorlinks=true,
    linkcolor=blue,
    filecolor=blue,
    citecolor=blue,
    urlcolor=cyan,
}
\usepackage[
backend=bibtex,
style=numeric-comp,
sorting=none,
maxbibnames=3,
citetracker=true,
natbib=true
]{biblatex}

\begin{document}

\begin{titlepage}
  \begin{sffamily}
  \begin{center}
    \textsc{\LARGE Université de Nantes}\\[1cm]
    \HRule \\[0.4cm]
    { \huge \bfseries On Computability, Learnability and Extractability of Finite State Machines from Recurrent Neural Networks \\[0.4cm] }
    \HRule \\[1cm]
\textbf{\huge{Reda Marzouk}} \\
 \vspace{2cm} 
 \Large{A disseration} \\
 submitted to the Department of Computer Science \\
 Nantes University \\
 In Partial Fulfilment of the requirements \\ 
 for the MSc degree \\ 
 \emph{Learning and Natural Language Processing}\\
 \vspace{1cm}
 \emph{Advisor:} Colin de la Higuera \\ \emph{Tutor:} Solen Quiniou
    \\[2cm]
    \vfill
    {\large \nth{3} July 2020}
  \end{center}
  \end{sffamily}
\end{titlepage}

\newpage

\begin{abstract}
This work aims at shedding some light on connections between finite state machines (FSMs) and Recurrent Neural Networks(RNNs). Examined connections in this master's thesis are threefold: The extractability of finite state machines from Recurrent Neural Networks, learnability aspects, and computationnal links. With respect to the former, the longstanding \emph{clustering hypothesis} of RNN hidden state space when trained to recognize regular languages  was explored, with new insights into this hypothesis through the lens of the generalization theory of Deep Learning are provided. As for learnability, an extension of the active learning framework better suited to the problem of approximating RNNs with FSMs is proposed, with the aim of better formalizing the problem of RNN approximation by FSMs. Theoretical analysis of two possible scenarios in this framework were performed. With regard to computability, new computational results on the distance and the equivalence problem between RNN trained as language models and different types of weighted finite state machines were given. \\
\end{abstract}
\vspace{1cm}
\renewcommand{\abstractname}{Résumé}
\begin{abstract}
Ce travail a pour but de mettre la lumiére sur les connexions existantes entre les machines à états finits et les réseaux de neurones reccurents. Les connexions examinées sont de trois ordres: L'extractabilité des machines à état fini à partir des réseaux de neurones recurrents, les aspects liés à l'apprenabilité, et les liens computationnelles. Concernant l'extractabilité, l'hypothése de \emph{partitionnement} de l'espace des états des RNNs lorsqu'ils sont entrainés à reconnaitre un langage régulier a été exploré, à la lumiére de la théorie de generalisation des modéles d'apprentissage profond. A propos de l'apprenabilité, une extension du cadre d'apprentissage actif plus convenable au probléme d'extraction des machines à états finis est proposée. L'analyse théorique de deux des scénarios possibls dans ce cadre d'apprentissage a été effectuée. Concernant la calculabilité, de nouveaux résultats autour de la distance et le probléme d'équivalence entre les RNNs entrainés comme modéle de langage et différents types de machines à états finis pondérés y sont présentés. 
\end{abstract}

\newpage 
\renewcommand{\abstractname}{Acknowledgments}
\begin{abstract}
\vspace{0.4cm}
$~~~$After five months of internship in LS2N within the TALN team in Nantes, three of which conducted remotely due to the COVID-19 epidemy, I couldn't begin this section without expressing my deepest gratitude to Colin de La Higuera for his support and his generosity during this difficult period. Besides his academic guidance, his fruitful recommendations about readings and paths to explore during our online discussions, I discovered a person with high human qualities. Surely, his personality has been inspiring to me.  \\
  
  $~~~$I would like to thank my dear friend Timothée for his help during this period, and his sense of dedication for others. Rarely have I met a person filled with such a strong inner obligation of solidarity toward others. \\ 
  $~~~$My thanks go also to Walid Ben Romdhane and Victor Connes for their supervision, and insighful advices during the "AI for the common good" Hackathon held in Paris. Without their encouragement from day one of the Hackathon till its end, I doubt we could have made our achievement by winning the WoW award of the Hackhathon. \\
  
  $~~~$Finally, although the quarantine due to the epidemy didn't give me a chance to familiarize with many members of the TALN team, I've still been capable to touch their kindness. I wish to see you soon in good health. 
\end{abstract}

\newpage

\tableofcontents

\section*{Introduction}
\label{chap:introduction}
   $~~~$Deep Learning (DL) represents one of the major breakthroughs Machine Learning Community has made during the last decades. Its baffling capacity to generate models fitting complex phenomena appears to be unlimited. Yet, putting aside the pragmatic argument of performance, there is one non-negligeable issue that DL still suffers from: the highly abstract nature of the way it represents knowledge. The incapacity of explaining decisions made by a predictive model raises concerns at different levels (ethical, societal, economic etc.). As long as this severe limitation is not resolved for DL, promising domains of application can't be reached by the DL technology. Although theoretically-oriented, the work presented in this thesis is fundamentally motivated by these aforementioned issues. \\ 
   $~~~$ This severe limitation of DL recently raised the attention of the DL community to develop tools and techniques to make DL models more transparent and their decisions more explainable. Those methods can be framed into two paradigms: the first targets the architecture of DL networks themselves where the ultimate hope is to design DL machines that boast both the outstanding predictive power of traditional architectures, and a high level of understandability \cite{Wang19}\cite{Zhang18}. The second tackled the problem in an \emph{ad-hoc} manner where the objective is to provide algorithmic and visualization tools that aims at shedding light on the type of knowledge encoded in already trained DL models designed to perform a specific task \cite{Karpathy16}\cite{Zeiler14}. The work developed in this thesis borrows its context from the latter paradigm. The ultimate goal of this work is to explore the perspective of designing \emph{compilers} capable of converting the \emph{"DL programs"} encoded in models' weight parameters into \emph{symbolic programs} encoded into strings of symbols easier to analyze, \emph{"debug"}, and process with classical tools of model checking used in software engineering. \\ 
    $~~~$ Toward this objective, our main focus in this thesis will be on Recurrent Neural Networks (RNNs), and their connections with Finite State Machines (FSMs). The angle of study will be threefold: (1) Computability properties of RNNs when projected to different types of Finite State Machines, (2) Exploring techniques to extract symbolic representations from Recurrent Neural Networks, (3) The last point, which is directly related to the former concerns the problem of learning finite state machines from trained Recurrent Neural models. If those three major properties are well understood, then a huge step toward the aforementionned utlimate goal would be already taken. \\
    $~~~$ This master's thesis is organized in four parts. Part I will provide the reader with general definitions and notions she/he needs to know to follow up the rest of this thesis. A general definition of a \emph{sequential machine} which encompasses all types of machines we'll see in this report will be given. Types of machines of interest in this work will be defined including weighted/binary finite state machines, and RNNs trained as recognizers of languages, or as language models. In part II, we will focus our attention on \emph{computationnal} properties of RNNs, and will derive results of computability and complexity of important problems connecting Weighted Finite State Machines (WFSAs) to RNNs trained as language models (RNN-LMs). In part III, we'll discuss methods explored in the literature examining perspectives of extracting finite state machines from RNNs. We'll analyze carefully the experimental setup designed to evaluate and compare the performance of algorithms designed to perform such task. In light of our experiments and new insights from the theory of generalization in Deep Learning theory, we shall give explanations of the longstanding \emph{"clustering effect hypothesis"} of RNNs trained to recognize regular languages. In the last part, we'll address the problem of learnability of Recurrent Neural Networks by Deterministic Finite State Automata. We'll propose an extension of the active learning framework that is best suited to our task. Two scenarios under this novel framework will be explored with derived sample complexity bounds.

\section{General Definitions and Notations}
\label{chap:example} 
In this part, we'll present general definitions and notions that will serve as a prerequisite for the rest of this master's thesis. A general definition of sequential machines, that encompasses all types of machines encountered in this work, will be given and precise definitions of machines that will carefully examined in the rest of this thesis will be provided, including Deterministic Finite Automata (DFA), Different types of weighted automata (WA) and different classes of RNNs. \\ 
In all this work, we'll borrow our terminology from formal language theory. However, we note that concepts and treatments presented here are general, and doesn't exclude, for instance, Natural Language Processing whose community often prefers to use a terminology closer to the field of linguistics. With that said, we present our general notation: \\
$\Sigma$ will denote any any-non finite finite alphabet. The set of all finite strings is denoted by $\Sigma^{*}$. The set of all strings whose size is equal (resp. greater than or equal) to $n$ is denoted by $\Sigma^{n}$(resp. $\Sigma^{\geq n}$). For any string $w \in \Sigma^{*}$, the size of $w$ is denoted by $|w|$, and its $n$-th symbol by $w_{n}$. The prefix of length $n$ for any string $w \in \Sigma^{\geq n}$ will be referred to as $w_{:n}$. The symbol $\$$ denotes a special marker. The symbol $\Sigma_{\$}$ will refer to the set $\Sigma \bigcup \{\$\}$.
\subsection{Definition of a general sequential machine}
 Let's begin this section by the following definition:
\begin{definition}
 A sequential machine M is defined by the triplet $M = <\Sigma,Q,q_{0},f,g,\mathbb{H}>$, where: 
 \begin{itemize}
     \item $\Sigma$ is the alphabet,
     \item $Q, \mathbb{H}$ are two arbitrary sets called respectively the state space and the output space,
     \item $q_{0} \in Q$ is called the initial state,
     \item $f: Q \times \Sigma \rightarrow Q$ is called the transition function,
     \item $g: Q \rightarrow \mathbb{H}$ is called the output function,
 \end{itemize}
 When $M$ runs on a string $w=w_{1}..w_{n}$, It produces the output: $$M(w) = g(f(w_{n},f(w_{n-1}..,f(w_{1},h))$$ 
 \end{definition}

This general definition covers many interesting classes of \emph{computational devices} employed to process languages, such as deterministic and non-deterministic finite state machines, pushdown autamata \footnote{The state space of a pushdown automaton is the cartesian product of a finite state space with the set of all strings that represents its stack content}, different variants of weighted automata and RNNs. It is worth mentioning that, for the sake of its generalizability, the definition given above doesn't give any aspects related to the computability of functions. Those issues are left to be adressed when a particular type of a suquential machine is defined. \\ 
In this work, we'll mainly encounter three types of sequential machines: RNNs, Deterministic finite automata (DFAs) and different variants of weighted automata(WA) including probabilistic deterministic finite state machines (PDFA), probabilistic finite automata (PFA) and classical weighted automata\footnote{Weighted automata are more rigorously defined over semi-rings \cite{mohri01}. In this work, whenever weighted automata are mentioned, they refer implictly to those defined on the traditional semi-ring.}. We will give in the rest of this section formal definitions of all these structures. But before going any further, It is important to mention that a sequential machine can be run on a word in a different manner than the way It's specified in the previous definition. For instance, alternative ways will consist at augmenting the machine with a one-way linear-bounded tape where the state reached by each prefix during the run is stored, then decodes the history of states to produce the output. This is a similar mechanism of how RNN Encoder-Decoder works. Another alternative that is of great interest in our work, since it mimicks the behovior of RNNs trained as language models which will be carefully examined in this work, is when the machine is augmented with a linear-bounded tape where the output produced on each prefix of the running word is stored, then the machine produces the final output by applying a function on the content of the tape. We note that sequential machines equipped with state history tape are at least as powerful as those equipped with an output history tape since they can easily simulate this latter by simply applying the output function on each element of the tape.
We give a definition of a one-way linear bounded state machine augmented with an output history tape:
\begin{definition}
 A sequential machine $M$ augmented with a one-way linear bounded tape called the \emph{output history tape} is a sequential machine $M = <\Sigma_{\$},Q,q_{0},f,g,\mathbb{H}>$ equipped with a linear-bounded tape $T$ that stores a sequence of elements of $\mathbb{H}$, and a function $g: \cup_{n=1}^{\infty} \mathbb{H}^{n} \leftarrow \mathbb{H}$. \\
 $M$ runs on input $w = \$ w_{1}..w_{n}$ as follows:\\
 $~~~~$ - At each iteration i,
     process $f(w_{i},f(w_{1,i-1}))$ and stores the result in the tape $T$, \\
 $~~~~$ - The final output of the function is given by $g(T)$:  
\end{definition}

\subsection{Finite State Machines}
\subsubsection{Deterministic Finite State Automata}
 Deterministic Finite State Automata (DFA) is a type of sequential machines that recognizes languages generated by regular grammars \cite{Hopcroft06}. It represents one of the simplest types of symbolic machines, as its expressive power is restricted to regular languages, the lowest class of languages in the Chomsky hierarchy. However, its efficient computational properties in terms of parsing strings and checking its properties in polynomial time (such as, the equivalence property between two DFA. More of this will be discussed in the next chapter) makes it widely used in model checking problems. Deterministic Finite State automata are defined formally as follows:  
\begin{definition}
 A deterministic finite state automaton, denoted $\mathcal{A}$, is a sequential machine $<\Sigma, Q,q_{0}, f,g,\mathbb{H}>$ where $Q$ is any finite set and $\mathbb{H} = \{0,1\}$. \\ 
 The size of the set $Q$ is called the size of the automaton, denoted also as $|\mathcal{A}|$.
\end{definition}
Throughout this thesis, the class of all deterministic finite automata will be denoted as $\mathcal{DFA}(\Sigma)$. The family of DFAs whose size is equal to $n$ will be denoted as $\mathcal{DFA}_{n}(\Sigma)$.
\subsubsection{The minimal DFA and the Nerode equivalence}
Before wrapping up this part about finite state automata, we need to discuss one of the most important characterizations of regular languages in terms of deterministic finite automata that we will encounter in many parts of this thesis: the notion of Nerode equivalence class  of a language.
\begin{definition}
   Let $L$ be a language defined over a finite alphabet $\Sigma$. The Nerode relation associated to $L$ is defined as:
   $$w \equiv v ~~ \iff~~ (\forall z \in \Sigma^{*}:~~ wz \in L \iff vz \in L) $$
\end{definition}
It's easy to show that this relation defines an equivalence relation. Myhill and Nerode gave a characterization of the class of regular languages in terms of the Nerode relation. We present the Myhill-Nerode theorem: 
\begin{theorem}{(Myhill-Nerode Theorem)}
 $L$ is a regular language if and only if the number of equivalence classes of its Nerode relation is finite. 
\end{theorem}
This theorem has an important implication on both automata theory and automata identification. With respect to this latter, classical rule-based algorithms to identify an automaton rely heavily on the notion of Nerode relation \cite{higu10}. As for the former, a direct corollary of the theorem is the existence of a canonical (unique) minimal automaton that recognizes a given regular language $L$. Each state of this canonical automaton maps directly to a Nerode equivalence class. In the rest of this thesis, we'll often use the term \emph{Nerode equivalence} of two prefixes, and It should be understood as two strings encoding a path that leads to the same state from the initial state of the canonical DFA.

\subsection{Weighted Finite Automata}
In this section, we'll give a brief description of weighted versions of finite automata. A general weighted finite automaton is defined as follows: 
\begin{definition}
 A weighted automaton of size $n \in \mathbb{N}$ is a parametrized sequential machine where: $Q= \mathbb{R}^{n}$, and $\forall \sigma \in \Sigma:~ f(q,\sigma) = A_{\sigma}q$ where $A_{\sigma} \in \mathbb{R}^{n \times n}$, and $g(q) = \beta^{T} Q$, where $\beta \in \mathbb{R}^{n}$.
\end{definition}
The class of weighted automata is parametrized by $\{q_{0},\{A_{\sigma}\}_{\sigma \in \Sigma},\beta\}$. To have a homogeneous notation with the litterature of weighted finite automata, we'll use the symbol $\alpha$, instead of $q_{0}$ when treating weighte automata in the remaining of this thesis. \\ 
  Besides this algebraic characterization, an intuitive way to think about WFAs is through its graphical representation: WFAs can be graphically represented as weighted versions of nondeterministic finite automata, where transitions between states, denoted $\delta(q,\sigma,q')$ where $q,q' \in Q$ represents states of the WFA are labeled with a rational weight $T(q,\sigma,q')$, and each of its nodes $q \in Q$ is labeled by a pair of rational numbers $(I(q),P(q))$ that represents respectively the initial-state and final-state weight of $q$. WFAs model  weighted languages where the weight of a string $w$ is equal to the sum of the weights of all paths whose transitions encode the string $w$. The weight of a path $p$ is calculated as the product of the weight labels of all its transitions, multiplied by the initial-state weight of its staring node and the final-state weight of its ending node. 
\subsubsection{Probabilistic Finite Automata}
     A probabilistic finite automaton(PFA) is a sub-class of WFA designed in a way to encode stochastic languages. For this, the initial state vector is constrained to represent a probability distribution (the initial probability distribution). And, each entry in the transition matrix $A_{\sigma}(i,j)$ represents the probability to emit a symbol $\sigma$ from a state $i$, then transition to state j. Interestingly, PFAs are proven to be equivalent to Hidden Markov Models (HMMs), and the construction of equivalent HMMs from PFAs and vice versa can be done in polynomial time \cite{Vidal05}. The formal definition of a PFA is given as follows: 
\begin{definition}
  A probabilistic finite state automaton of size $n$ is a weighted automaton with the following constraints:  
  \begin{itemize}
      \item $\sum_{i=0}^{n} q_{0}[i] = 1 $, and $\forall i \in [n]:~q_{0}[i] \geq 0$,
      \item Each transition matrix is stochastic, that is $\forall{i} \in [n]:~\sum_{j=1}^{n} A[i,j] = 1$, and $A[i][j] \geq 0$
  \end{itemize}
\end{definition}
Again, to gain intuitions about this definition, we return to the graphical representation of a PFA:  A PFA is a WFA with two additional constraints: First, the sum of initial-state weights of all states is a valid probability distribution over the state space. Second, for each state, the sum of weights of its outcoming edges added to its finite-state weight is equal to $1$. This additional constraint restricts the power of PFAs to encode probabilistic languages \cite{Thollard05}, which makes it useful for representing language models. Interestingly, PFAs are proven to be equivalent to Hidden Markov Models (HMMs), and the construction of equivalent HMMs from PFAs and vice versa can be done in polynomial time \cite{Vidal05}. The deterministic version of PFAs, a.k.a \textbf{Deterministic Probabilistic Finite Automata (DPFA)}, enforces the additional constraint that for any state $q$, and for any symbol $\sigma$ there is at most one outgoing transition labeled by $\sigma$ from $q$. \\

\subsection{Recurrent Neural Networks}
Recurrent  Neural  Networks  and  their  different  variants represent an important family of Deep Learning models suitable to learning   tasks with sequential data. Generally, in language processing, there are three ways to train and use Recurrent Neural Networks in linguistic tasks: 
\begin{itemize}
    \item \textbf{RNNs as language recognizers:} In this type of tasks, RNNs are trained as a classifier of strings. Examples of using RNNs as language recognizers include Sentiment Analysis \cite{bakhta17}, text classification \cite{liu16} etc. 
    \item \textbf{RNNs as language Model (RNN-LM):} RNNs trained as language models assigns a weight that accounts for the importance of a word in the language. Their counterpart in the finite state automata world are Weighted Finite Automata and different variants,
    \item \textbf{RNNs as Seq2Seq models:} RNNs functioning as Seq2Seq models take as an input a sequence of symbols/words and output another sequence of symbols/words. Their analog counterpart in the world of symbolic machines are transducers. They find their applications in many language tasks such as Machine Translation\cite{mahata18}\cite{barone17}, speech processing \cite{graves13}    
\end{itemize}
In this work, we shall mainly focus on the first two families. We'll leave the Seq2Seq case for furture research. \\ 
The difference between those types of RNN machines lies fundamentally in their architecture and how they process data, and not in the type of the used transition function (also called RNN cells in the litterature of RNNs). Following our previous definitions of sequential machines, RNNs trained as language recognizers belong to the family of simple sequential machines given in definition 1.1.1. RNNs trained as language models are those augmented with a one-way linear bounded output tape (definition 1.1.2). \\
We shall present in the rest of this section different types of traditionally used RNN cells. We categorize them into two families: simple RNN cells, and gated RNN celles. 
\subsubsection{Simple first and second-order RNN cells}
Simple RNN cells are the oldest type of cells employed for Recurrent Neural Networks. In the family of simple RNN cells, there are generally two main sub-families of cells: First-order and Second-order. However, we draw the reader's attention that some works proposed higher-order types of RNNs \cite{soltani16}. But the first and second order remains the most widely chosen in practice. Due to their simplicity and better \emph{"trainability"} capacities, simple first-order RNN cells are more privileged in practice than second-order ones. \\ 

$\bullet$ \textbf{First-order RNN cells:} The transition function of a simple first-order RNN cell takes generally the following form: 
\begin{equation*}
       q_{t+1} = \phi(Wq_{t}+ b_{\sigma} + c),~~~ q_{t},~b_{\sigma},~c \in \mathbb{R}^{d},~W \in \mathbb{R}^{d \times d}
\end{equation*}
where $\phi$ is the activation function that determines the type of the cell, and $W,b$ are the parameters of the cell. $W$ is called the transition matrix, and $b_{\sigma}$ is the embedding vector of the input symbol $\sigma$ being injected to the cell, and $c$ is a bias vector. The hyper-parameter $d$ represents the dimension of the hidden state space, and determines the architecture of the network. 

$\bullet$ \textbf{Second-order RNN cells:} The main difference between first-order and second-order RNN cells resides in the type of interaction between neurons of the hidden state vector, and neurons of the input symbol. In a first-order RNN, the interaction of the input symbol with the hidden state vector is additive, while in the second-order, It's multiplicative. The general form of a second-order RNN cell is given as: 
\begin{equation*}
    q_{t+1}[i] = \phi(\sum\limits_{j,k} W_{ijk} q_{t}[j]b_{\sigma}[k]),~~~W \in \mathbb{R}^{d \times d \times L},~q_{t} \in \mathbb{R}^{d},~b_{\sigma} \in \mathbb{R}^{L}
\end{equation*}
where $\phi$ is the activation activation, $L$ is the size of the embedding space of symbols and $W$ is a tensor of order 3 that captures the interaction between each neuron of the hidden state vector $q_{t}$, and each neuron of the injected input symbol $\sigma$.

\textbf{Note on activation functions.} \emph{Historical works about RNN machines were dedicated to the analysis of RNNs using simple threshold functions, or linear-saturated functions \cite{Siegelmann95}. However, in practice, this type of activation function suffers from severe non-differentiability issues which makes it harder(or, even impossible in the case of the hard threshold function) to train with a gradient-descent type of algorithm. Nowadays, the most widely used activation functions are: The sigmoid function,~the hyperbolic tangent function and the rectified linear unit function (ReLu) function. For the sake of completeness, we provide the general expression for each of these activation functions: 
\begin{itemize}
    \item The sigmoid function: $\phi(x) = \frac{1}{1 + \exp(-x)}$,
    \item The hyperbolic tangent function: $\phi(x) = \frac{\exp(x) - \exp(-x)}{\exp(x) + \exp(-x)}$,
    \item The Rectified Linear Unit function: $\phi(x) = \max(0, x)$
\end{itemize} 
}
\subsubsection{RNN cells with gated memory}
A severe issue from which suffer RNNs equipped with simple cells is the difficulty of training with long strings, where It fails to capture long-range dependencies \cite{bengio93}. Authors in \cite{hochreiter97} proposed a solution to this problem by equipping RNN cells with memory gates capable of capturing long-range dependencies. The two most widely known types of RNN cells that follow this paradigm of RNN architecture design are LSTMs \cite{hochreiter97}, and GRUs \cite{cho14}. \\
$\bullet$ \textbf{LSTM cells.} An LSTM cell is equipped with three types of gates: the input gate, the forget gate and the output gate. The intuition behind using those cells is that each of which is responsible of controlling the amount of information that flows in the network as it runs on a particular string. The \textit{input gate} is responsible of deciding on the amount of information the cell will capture about the input. The \textit{forget gate} controls the amount of information the network will capture about the past. And the \textit{output gate} controls the information that will be transmitted to the cell output. The general equations of an LSTM cell running on a symbol $\sigma$ are given as follows: 
\begin{equation*}
    \begin{aligned}
     i_{t+1} &= \phi(W_{i}q_{t} + U_{i}b_{\sigma}),&~~(The~~input~~ gate), \\
     f_{t+1} &= \phi(W_{f}q_{t} + U_{f}b_{\sigma}),&~~(The~~forget~~ gate), \\
     o_{t+1} &= \phi(W_{o}q_{t} + U_{o}b_{\sigma}),&~~(The~~output~~ gate), \\
    \end{aligned}
\end{equation*}
where $\phi$ is the sigmoid activation function. \\ 
After the update of the memory gates, their values are aggregated by a memory cell $c_{t}$, and the hidden state vector flowing out of the LSTM cell is constructed using the following expressions:
\begin{equation*}
    \begin{aligned}
      c_{t+1} &= c_{t} \otimes f_{t} + tanh(W.b_{\sigma} + U.q_{t}) \otimes i_{t} \\
      h_{t+1} &= tanh(c_{t} \otimes o_{t})
    \end{aligned}
\end{equation*}

$\bullet$ \textbf{GRU cells.} The architecture of the LSTM is highly complex, which makes it difficult to train efficiently. In an attempt to simplify the architecture, while still keeping the property of capturing long-range dependencies, authors in \cite{cho14} proposed a \emph{simplified} version architecture of cells with memory gates, called GRU. This architecture uses only two gates, the \textit{reset gate} denoted $r_{t}$, and the \textit{update gate} denoted $z_{t}$. The general equations of a GRU cell is given as:
\begin{equation*}
    \begin{aligned}
    z_{t+1} &= \phi(W_{z}q_{t} + U_{z}b_{\sigma}) \\
    r_{t+1} &= \phi(W_{r}q_{t} + U_{r}b_{\sigma}) \\
    q_{t+1} &= (1-z_{t+1})(\tanh(W (q_{t} \otimes r_{t+1}) + b_{\sigma})  + z_{t+1}()
    \end{aligned} 
\end{equation*}

\subsection{Conclusion}
 The objective of this section was to make the reader familiar with notions treated in the rest of this thesis. Definitions and explanations of machines, either finite state machines or RNN machines, were provided. In the following, those objects will receive examination from three angles: the computational aspect, the extractability of one from the other (in our case, FSMs from RNNs), and learnability. In the next section, we shall focus on the first aspect, the computational one. We'll address the question of computing the distance and the equivalence between particular families of finite state machines and Recurrent Neural Networks.

\section{Computational problems between Weighted Finite State Machines and Recurrent Neural Networks trained as Language Models}
 $~~~$Whenever a new class of computational devices that recognize languages is proposed in the literature, it is usually analyzed from two aspects: its properties from the angle of formal language theory, and its computational properties. The former refers to its expressiveness power projected to known classes of formal languages, its closure properties etc.. The latter focuses more on its practical usefulness. A perfect computational device would be one that is both highly expressive and efficiently \emph{"evaluatable"} when simulated on a computer. Unfortunately, it's usually the case that some trade-off needs to be made. To illustrate this point, we take the example of Deterministic Finite State Automata: Finite State Automata are known to have weak expressive power when projected to the classical Chomsky hierarchy of formal languages. Yet, it's still one of the most widely privileged \emph{computational devices} (or, one of its variants) for model checking in software engineering. The fact is that in this particular practical case, efficient evaluation is the most important requirement. \\
 $~~~$One of the most important computational problems that computer scientists address about any kind computational device is the equivalence problem. For instance, It is known that checking equivalence between deterministic finite automata can be done in polynomial time, while the non-deterministic counterpart is PSPACE-Complete \cite{Hopcroft06}. When it comes to different classes of Weighted Automata, due to the quantitative nature of language they recognize, other questions arise such as the complexity of measuring distance between quantitative languages they process, the optimization problem related to find the highest weighted string.\\ 
 $~~~$ As we already stated in the introduction, the ultimate goal of this work is to design \emph{compilers} of recurrent neural networks to finite state machines. If such is the goal, then a question arises about the computational complexity of evaluating equivalence and/or distance between an RNN and an FSM. In this section, we'll provide some theoretical insights on this issue. Results presented in this section are mainly taken from our paper that can be found in \cite{marzouk20}. \\
 Our main focus in this section will be on quantitative languages. More precisely, we will restrict our analysis on the class of first-order RNNs with $ReLu$ cells trained as language models (RNN-LMs) and different classes of weighted automata (PDFAs/PFAs/WFAs). The choice of $ReLu$ is not arbitrary. In fact, due to its nice piecewise-linear property and its wide use in practice, the $ReLu(.)$ function is a first choice to analyze theoretical properties of RNN architectures. Analyzing the case of RNNs with highly non-linear activation functions (e.g. the sigmoid, the hyperbolic tangent etc.) is left for future research. \\ 
  In the first part of this section, we'll present a formal definition of a RNN machine trained as a language model with $ReLu$ as an activation function. Since the treatment of RNNs in this part of the thesis is from the computational point of view, this definition will differ from the classical algebric description of a Recurrent Neural Network  generally adopted in the literature. In the next part, we'll give insights into the construction that will serve us to prove many results in the following section. Afterwards, we shall give computational results of three problems: equivalence between RNN-LMs and PDFAs/PFAs/WFAs, computing the distance between them, and deciding if a PFA approximates well a first-order RNN-LM with $ReLu$ used as an activation function.
  
  \subsection{Definition of a first-order RNN-LM as a computational model}
  A formal definition of a first-order RNN-LM from a computational angle is given as follows:
  \begin{definition}\cite{Chen18}
 A First-order weighted RNN Language model is a weighted language $f: \Sigma^{*} \rightarrow \mathbb{R}$ and is defined by the tuple $<\Sigma, N, h^{(0)}, \sigma, W, (W')_{\Sigma_{\$}}, E, E'>$ such that:
 \begin{itemize}
     \item $\Sigma$ is the input alphabet,
     \item $N$ the number of hidden neurons,
     \item $\sigma: \mathbb{Q} \rightarrow \mathbb{Q}$ is a computable activation function, 
     \item $W \in \mathbb{Q}^{N \times N}$ is the state transition matrix,
     \item $\{W'_{\sigma}\}_{\sigma \in \Sigma_{\$}}$, where each $W'_{\sigma} \in \mathbb{Q}^{N}$ is the embedding vector of the symbol $\sigma \in \Sigma_{\$}$,
     \item $O \in \mathbb{Q}^{\Sigma_{\$} \times N}$ is the output matrix,
     \item $O' \in \mathbb{Q}^{\Sigma_{\$}}$ the output bias vector.
 \end{itemize}
 The computation of the weight of a given string $w$ (where $\$$ is the end marker) by $R$ is given as follows. \\
 (a) Recurrence equations:
 $$ h^{(t+1)} = \sigma(W.h^{(t)} + W'_{w_{t}})$$
 $$ E_{t+1} = O h^{(t+1)} + O'$$
 $$ E'_{t+1} = softmax_{2}(E_{t+1})$$
 (b) The resulting weight:
 $$ R(w) = \prod\limits_{i=0}^{|w|+1} E'_{i}$$
 where $w_{0} = w_{|w|+1} = \$$
\end{definition}
Notice that, in order to avoid technical issues, we used softmax base 2 defined as: $softmax_{2}(x)_{i} = \frac{2^{x_{i}}}{\sum\limits_{j=1}^{n} 2^{x_{j}}}$ for any $x \in \mathbb{R}^{d}$ instead of the standard softmax in the previous definition.  
In the following, hidden units of the network will be designated by lowercase letters $n_{1},n_{2},..$, and their activations at time $t$ by $h_{n}^{t}$. Also, we denote by $\mathcal{R}_{\sigma}$ the class of RNN-LMs when $\sigma$ is the activation function. For example, an important class of RNN-LMs that will be used extensively in the following development of this section is $\mathcal{R}_{ReLu}$. In terms of complexity, we assume that parameters of a network $R$ is finite-precision\footnote{In all RNN constructions presented in the rest of this chapter, we use only finite-precision RNNs, which justifies our assumption.}, in which case the size of a network $R$ will be equal to $O(N
^{2})$, where $N$ is the number of hidden neurons. 
 \subsection{On Turing Completeness of Recurrent Neural Networks}
 The main tool for proving results in this section relies on the Turing Completeness of the class of Recurrent Neural Networks with $ReLu$ as an activation function proved by Siegelemann et al. in \cite{Siegelmann95}. We dedicate this section to give details about the construction, and we will end this section by giving a characterization of the halting problem\footnote{The halting problem is defined as follows: Given a Turing Machine M, and a string $w$, decide whether the machine $M$ halts on $w$ is undecidable.} that connects it to the class of RNN-LMs. This result will be used in the next section to prove results presented there. \\ 
 The main intuition of Siegelmann \textit{et al.}'s work is that, with an appropriate encoding of binary strings, a first-order RNN with a saturated linear function can readily simulate a stack data structure by making use of a single hidden unit. For this, they used 4-base encoding scheme that represents a binary string $w$ as a rational number: $Enc(w) = \sum\limits_{i=1}^{|w|} \frac{w_{i}}{4^{i}}$. Backed by this result, they proved than any two-stack machine can be simulated by a first-order RNN with linear saturated function, where the configuration of a running two-stack machine (i.e. the content of the stacks and the state of the control unit) is stored in the hidden units of the constructed RNN.  Finally, given that any Turing Machine can be converted into an equivalent two-stack machine (the set of two-stack machines is Turing-complete \cite{Hopcroft06}), they concluded their result. \\ 
 In the context of our work, two additional remarks need to be made about Siegelmann's construction: - First, although the class of first-order RNNs examined in their work uses the saturated linear function as an activation function, as raised in \cite{Chen18}, their result is generalizable to the ReLu activation function (or, more generally, any computable function that is linear in the support [0,1])? - Second, although not mentioned in their work, the construction of the RNN from a Turing Machine is polynomial in time. In fact, on one hand, the number of hidden units of the constructed RNN is linear in the size of the Turing Machine, and the construction of the transition matrices of the network is also linear in time. On the other hand, notice that the 4-base encoding map $Enc(.)$ is also computable in linear time. \\ 
 In light of these remarks, we are now ready to present the following theorem:
 \begin{theorem}({Theorem 2, \cite{Siegelmann95})}
   Let $\phi: \{0,1\}^{*} \rightarrow \{0,1\}^{*}$ be any computable function, and $M$ be a Turing Machine that implements it. We have, for any binary string $w$, there exists $N = O(poly(|M|)),~h^{(0)} = [Enc(w)~~0..0] \in \mathbb{Q}^{N},~W \in \mathbb{Q}^{N \times N}$, such that for any finite alphabet $\Sigma$, $~\forall \sigma \in \Sigma_{\$}:~W'_{\sigma} \in \mathbb{Q}^{N}, O \in \mathbb{Q}^{|\Sigma_{\$}| \times N},~ O' \in \mathbb{Q}^{|\Sigma_{\$}|}$, $R = <\Sigma,N,ReLu,W,W',O,O'> \in \mathcal{R}_{ReLu}$
    verifies:
   \begin{itemize}
       \item if $\phi(w)$ is defined, then there exists  $T \in \mathbb{N}$ such that the first element of the hidden vector $h_{T}$ is equal to $Enc(\phi(w))$, and the second element is equal to $1$,
       \item if $\phi(w)$ is undefined (i.e. $M$ never halts on $w$), then for all $t \in \mathbb{N}$, the second element of the hidden vector $h_{t}$ is always equal to zero.  
   \end{itemize}
   Moreover, the construction of $h_{0}$ and $W$ is polynomial in $|M|$ and $|w|$. 
 \end{theorem}
 In the following, we'll denote by $\mathcal{R}_{ReLu}^{M,w}$ the set of RNNs in $\mathcal{R}_{ReLu}$ that simulate the TM $M$ on $w$. It is important to note that the construction of a RNN that simulates a TM on a given string in the previous theorem is both input and output independent. The only constraints that are enforced by the construction are placed on a block of the transition matrix of the network, and the initial state. In fact, the input string is \textit{placed} in the first stack of the two-stack machine before running the computation (i.e. in the initial state $h^{(0)}$). Under this construction, the first stack of the machine is encoded in the first hidden unit of the network. Afterwards, the \textit{RNN Machine} runs on the empty string, and halts (If It ever halts) when the halting state of the machine is reached. In Theorem 1.1, the halting state of the machine is represented by the second neuron of the network. In the rest of this section, we'll refer to the neuron associated to the halting state by the name \textit{halting neuron}, denoted $n_{halt}$. \\
  We present the following corollary that gives a characterization of the halting machine problem\footnote{The Halting Machine problem is defined as follows: Given a TM M and a string w, does M halt on w? This problem is undecidable.} that relates it to the class $\mathcal{R}_{ReLu}$:
 \begin{corollary}
   Let $M$ be any Turing Machine, and $w$ be a binary string, $M$ halts on $w$ if and only if for any $R \in \mathcal{R}_{Relu}^{M,w}$ , there exists $T \in \mathbb{N}$, such that $\forall t < T: h_{n_{halt}}^{(t)}=0$, and $h_{n_{halt}}^{(T)}=1$.
 \end{corollary}
 
 \subsection{The equivalence and distance problem between WFSM/RNN-ReLu in the general case}
 In this section, we will be interested in the problem of deciding equivalence and computing distances between general first-order RNN-LMs with ReLu used as an activation function and PDFA/PDFA/WFA. 
 The equivalence problem between a DPFA and a general RNN-LMs is formulated as follows: \\
 \noindent \textbf{Problem.} Equivalence Problem between a DPFA and a general RNN \\
 \textit{Given a general RNN-LM $R \in \mathcal{R}_{ReLu}$ and a DPFA $\mathcal{A}$. Are they equivalent?}
 \begin{theorem}
 The equivalence problem between a DPFA and a general RNN is undecidable
 \end{theorem}
 \begin{proof}
 We reduce the halting Turing Machine problem to the Equivalence problem.
   Let $\Sigma=\{a\}$. We first define the trivial DPFA $\mathcal{A}$ with one single state $q_{0}$, and $T(\delta_{q_{0},a,q_{0}})=P(q_{0})=\frac{1}{2},~I(q_{0})=1$. This DPFA implements the weighted language $f(a^{n})=\frac{1}{2^{n+1}}$. \\
   Let $M$ be a Turing Machine and $w \in \Sigma^{*}$. We construct $R \in \mathcal{R}_{ReLu}^{M,w}$ such that $O[n_{halt},a]=1$ ,0 everywhere and $O'$ is equal to zero everywhere. We build another RNN $R'$ from $R$ by adding one neuron in its hidden layer, denoted $n'$ such that: $h_{n'}^{(0)} = 0,~\forall t \geq 0:~h_{n'}^{(t+1)}= ReLu(h_{n'}^{(t)}),~ O[n',\$]=1$. \\
   Notice that, by Corollary 4.2, the TM $M$ never halts on $w$ if and only if $\forall T: (h_{n_{halt}}^{(T)},h_{n'}^{(T)})=(0, 0)$, i.e. $R'(a^{n}) = \frac{1}{2^{n+1}}$. That is, the TM $M$ doesn't halt on $w$ if and only if the DPFA $\mathcal{A}$ is equivalent to $R'$, which completes the proof.
 \end{proof}
 
  A direct consequence of the above theorem is that the equivalence problem between PFAs/WFAs and general RNN-LMs in $\mathcal{R}_{ReLu}$ is also undecidable, since the DPFA problem case is immediately reduced to the general case of PFAs (or WFAs). Another important consequence is that no distance metric can be computed between  DPFA/PFA/WFA and $\mathcal{R}_{ReLu}$:
 
 \begin{corollary}
   Let $\Sigma = \{a\}$. For any distance metric $d$ of $\Sigma^{*}$, the total function that takes as input a description of a PDFA $\mathcal{A}$ and a general RNN-LM $\mathcal{R}_{ReLu}$ and outputs $d(\mathcal{A}, R) $ is not recursive. \\
   This fact is also true for PFAs and WFAs.
 \end{corollary}
\begin{proof}
  The proof relies on the properties of distance metrics.
  Let $d$ be any distance metric on $\Sigma^{*}$. By definition of a distance, we have $d(\mathcal{A}, R) = 0$ if and only if $\mathcal{A}$ and $R$ are equivalent. Since the equivalence problem is undecidable, $d(.)$ can't be computed.
\end{proof}

\subsubsection{Intersection of the cut language of a general RNN-LM with a DFA}
 In this subsection, we are interested in the following problem: \\
\noindent\textbf{Problem.} Intersection of a DFA and the cut-point language of a general RNN-LM  \\
 \textit{ Given a general RNN-LM $R \in \mathcal{R}_{ReLu}$, $c \in \mathbb{Q}$, and a DFA $\mathcal{A}$, is $\mathcal{L}_{R,c} \bigcap \mathcal{L}_{\mathcal{A}} = \emptyset$?}
 Before proving that this problem is undecidable, we shall recall first a result proved in \cite{Chen18}:
 \begin{theorem}{(Theorem 9, \cite{Chen18})}
 Define the highest-weighted string problem as follows: Given a RNN-LM $R \in \mathcal{R}_{ReLu}$, and $c \in (0,1)$: Does there exist a string $w$ such that $R(w)>c$? \\
 The highest-weighted  string problem is undecidable.
 \end{theorem} 
 \begin{corollary}
 The intersection problem of a DFA and the cut-point language of a general RNN-LM is undecidable.
 \end{corollary} 
 \begin{proof}
 We shall reduce the highest-weighted string problem from the intersection problem. Let $R \in \mathcal{R}_{ReLu}$ a general weighted RNN-LM, and $c \in (0,1)$. Construct the automaton $\mathcal{A}$ that recognizes $\Sigma^{*}$. We have that $\mathcal{L}_{\mathcal{A}} \bigcap \mathcal{L}_{R} = \mathcal{L}_{R} = \emptyset$ if and only if there exist no string $w$ such that $R(w)>c$, which completes the proof.
 \end{proof}
 
 \subsubsection{The equivalence problem over finite support}
 Given that the equivalence problem between a general RNN-LM and different classes of finite state automata is undecidable, a less ambitious goal is to decide whether a RNN-LM agrees with a finite state automaton over a finite support. We formalize this problem as follows: \\
 \textbf{Problem.} The EQ-Finite problem between PDFA and general RNN-LMs \\
 \textit{ Given a general RNN-LM $R \in \mathcal{R}_{ReLu}$, $m \in \mathbb{N}$ and a PDFA $\mathcal{A}$. Is $R$ equivalent to $\mathcal{A}$ over $\Sigma^{\leq m}$? }
 
 \begin{theorem}
  The EQ-Finite problem is EXP-Hard.
 \end{theorem}
 \begin{proof}
 We reduce the bounded halting problem \footnote{The bounded halting problem is defined as follows: Given a TM M, a string $x$ and an integer $m$, encoded in binary form. Decide if M halts on $x$ in at most $n$ steps? This problem is EXP-Complete.} to the EQ-Finite problem. \\
 The proof is similar to the used for Theorem 4.3.  We are given a TM $M$, a string $w$ and $m \in \mathbb{N}$. Let $\Sigma = \{a\}$. We construct a general RNN-LM $R'$ by augmenting $R \in \mathcal{R}_{ReLu}^{M,w}$ with a neuron $n'$ as in Theorem 4.3. By Theorem 4.1, this reduction runs in polynomial time. On the other hand, let $\mathcal{A}$ be the trivial PDFA with one single state $q_{0}$, and $T(\delta_{q_{0},a,q_{0}})=P(q_{0})=\frac{1}{2},~I(q_{0})=1$. Note that $R'$ doesn't halt in $m$ steps if and only if $\forall T \leq m:~(n_{halt}^{(T)},n'^{(T)})=(0, 0)$, i.e. $R'(a^{n}) = \frac{1}{2^{n+1}}$ for the first $m$ running steps on $R'$, in which case the language modelled by $R'$ is equal to $f$ in $\Sigma^{\leq m}$. Hence, $\mathcal{A}$ is equivalent to $R$ in $\Sigma^{\leq m}$ if and only if $M$ doesn't halt on the string $w$ in less or equal than $m$ steps.
 \end{proof}
\subsection{Approximate distance between PFA and RNN-LMs}
In the previous section, we have seen that problems related to equivalence and distance in the general case turned out to be either undecidable, or intractable when restricted to finite support. In this section, we examine the case where trained RNN-LMs are guaranteed to be consistent\footnote{It was proven only recently that RNN-LMs with ReLu activation function are not necessarly consistent, and deciding consistency is undecidable \cite{Chen18}. Characterizing consistency of different classes of RNN-LMs is still an open problem.}, and we raise the question of approximate equivalence between PFAs and first-order consistent RNN-LMs with \textit{general} computable activation functions.  It's worth noting that all results holding here for PFAs remain valid for Hidden Markov Models (HMM), since HMMs and PFAs are proved to be equivalent and there exists polynomial time algorithms to convert one to another and vice versa (See Propositon 4-5, \cite{Vidal05}). For any computable activation function $\sigma$, we formalise this question in the following two decision problems:  \\
\noindent\textbf{Problem.} \textit{Approximating the Tchebychev distance between RNN-LM and PFA} \\
\textbf{Instance:} A consistent RNN-LM $R \in \mathcal{R}_{\sigma}$, a PFA $\mathcal{A}$, $c>0$ \\
 \textbf{Question:} Does there exist $|w| \in \Sigma^{*}$ such that $|R(w) - \mathcal{A}(w)| > c$ \\

\noindent\textbf{Problem.} \textit{Approximating the Tchebychev distance between consistent RNN-LM and PFA over finite support} \\
\textbf{Instance:} A consistent RNN $R \in \mathcal{R}_{\sigma}$, a PFA $\mathcal{A}$, $c>0$ and $N \in \mathbb{N}_{+}$, \\
 \textbf{Question:} Does there exist $w \in \Sigma^{\leq N}$ such that $|R(w) - \mathcal{A}(w)| > c$ \\
 Note that there is no constraint on the activation function used for consistent RNN-LMs in these defined problems, provided it is computable. The first fact is easy to prove:
 \begin{theorem}
 Approximating the Tcheybechev distance between RNN-LM and PFA is decidable.
 \end{theorem}
 \begin{proof}
  Let $R$ be a consistent RNN-LM and $\mathcal{A}$ be a PFA. An algorithm that can decide this problem runs as follows: enumerate all strings $w_{1},..w_{t},..$ in $\Sigma^{*}$ until we reach a string that satisfies this property in which case the algorithm returns Yes. If there is no such string, by definition of consistency, there will be a finite time $T$\footnote{$T$ can be determined while running the algorithm through summing the probabilities of all reached strings in the enumeration} such that $\sum\limits_{t=1}^{T} R(w_{t}) \geq 1-c,~ \sum\limits_{t=1}^{T} \mathcal{A}(w_{t}) \geq 1-c$ in which case, we have: $\forall t >T:~ R(w_{t}) < c$ and $\mathcal{A}(w_{t}) < c$ which implies $\forall t > T: |R(w_{t} - \mathcal{A}(w_{t})| < c$. When $T$ is reached, the algorithm returns No.
 \end{proof}
 \subsubsection{Approximating the Tcheybetchev distance over a finite support}
 The rest of this section is dedicated to the proof of the NP-Hardness result. The proof will rely on a reduction from the 3-SAT problem \cite{goldreich10}. And, we'll derive below the construction of a PFA and a RNN from a given 3-SAT formula which will help us prove the result. \\ 
  A 3-SAT formula will be denoted by the symbol $F$. A formula is comprised of $n$ Boolean variables denoted $x_{1},..x_{n}$, and $k$ clauses $C_{1},..C_{k}$. For each clause, we'll use notation $l_{i1},~l_{i2},~l_{i3}$ to refer to its three composing literals. For a given string $w \in \{0,1\}^{n}$, the number of clauses satisfied by $w$ will be denoted by $N_{w}$. \\ 
  Without loss of generality, we assume in the sequel that literals $\{l_{i1},l_{i2},l_{i3}\}$ of a given clause $C_{i}$ are arranged in the order of atoms from which they derive, and we denote by $i_{C_{i}}^{*}$ the index of the atom of $l_{i1}$. Let $\epsilon \in (0,\frac{1}{2})$ whose value will be specified later. \\
$\bullet$ \textbf{Construction of a PFA $\mathcal{A}$: } the construction of our PFA is inspired from the work done in \cite{Casacuberta00}, and illustrated in Figure 1. Intuitively, each clause $i$ in $F$ is represented by two paths in the PFA, one that encodes a satisfiable assignment of the variables for this clause, and the other not. More formally, the PFA $\mathcal{A}$ is defined as:
\begin{itemize}
    \item $Q_{\mathcal{A}} = \{q_{0}\} \cup \{q_{ij}^{c}:~ i \in [1,k],~j \in [1,n],~ c \in \{T,F\}  \}$ is the set of states,
    \item Initial-state weights: $I_{\mathcal{A}}(q_{0})=1,~0$
    otherwise,
    \item Final-state weights: 
    \begin{itemize}
         \item For each clause $i$: $P_{\mathcal{A}}(q_{in}^{N}) = 1 - 2\epsilon$
         \item All the other states in $\mathcal{A}$ has a final-state probability equal to $2 \epsilon$
     \end{itemize}
\end{itemize}
$\bullet$ Transitions: For each clause $C_{i}$,
\begin{itemize}
     \item $\forall S \in \{T,F\},~a \in \Sigma$: $(q_{0},a,q_{i,1}^{S}) = \frac{1}{2k} - \frac{\epsilon}{k}$
     \item If $i_{C_{i}}^{*} \neq 1$: 
     \begin{itemize}
         \item If $l_{i1} = x_{i_{C_{i}}^{*}}$: \\ 
         $\forall S \in \{T,F\}:~T_{\mathcal{A}}(q_{i_{C_{i}}^{*}-1}^{S},1,q_{i_{C_{i}}^{*}}^{T})= \frac{1}{2} - \epsilon$ and \\ $T_{\mathcal{A}}(q_{i_{C_{i}}^{*}-1}^{S},0,q_{i_{C_{i}}^{*}}^{F}) = \frac{1}{2}- \epsilon$
         \item else: $\forall S \in \{T,F\}:~T_{\mathcal{A}}(q_{i_{C_{i}}^{*}-1}^{S},0,q_{i_{C_{i}}^{*}}^{T})= \frac{1}{2} - \epsilon$ and  $T_{\mathcal{A}}(q_{i_{C_{i}}^{*}-1}^{S},1,q_{i_{C_{i}}^{*}}^{F}) = \frac{1}{2}- \epsilon$
         \item If $i_{C_{i}}^{*} > 2$, then: \\
         $\forall 1 \leq i < i_{C_{i}}^{*} - 1,~a \in \Sigma,~S \in \{T,F\}:$, we have: $T_{\mathcal{A}}(q_{i}^{S},a,q_{i+1}^{S}) = \frac{1}{2} - \epsilon $
     \end{itemize}
     \item Else:
     \begin{itemize}
         \item If $l_{i1} = x_{1}$:
         $\forall a \in \Sigma:~ T_{\mathcal{A}}(q_{i,1}^{T},a,q_{i,2}^{T}) = \frac{1}{2} - \epsilon$. And: 
         \begin{itemize}
             \item If $x_{2} = l_{i2}$: then $T_{\mathcal{A}}(q_{i1}^{F},1,q_{i2}^{T}) = \frac{1}{2} - \epsilon$, and $T_{\mathcal{A}}(q_{i1}^{F},0,q_{i2}^{F}) = \frac{1}{2} - \epsilon$,
             \item If $l_{i2}=\bar{x}_{2}$: then $T_{\mathcal{A}}(q_{i1}^{F},0,q_{i2}^{T}) = \frac{1}{2} - \epsilon$, and $T_{\mathcal{A}}(q_{i1}^{N},1,q_{i2}^{N}) = \frac{1}{2} - \epsilon$
             \item Otherwise $\forall a \in \Sigma: T_{\mathcal{A}}(q_{i1}^{F},a,q_{i2}^{F})= \frac{1}{2} - \epsilon$
         \end{itemize}
         \item else:
          $\forall a \in \Sigma:~(q_{i1}^{F},a,q_{i2}^{T}) \in \delta_{\mathcal{A}}$, and:
         \begin{itemize}
             \item If $x_{2} \in \{l_{i1},~l_{i2},~l_{i3}\}$, then $T_{\mathcal{A}}(q_{i1}^{T},1,q_{i2}^{T}) = \frac{1}{2} - \epsilon$, and $(q_{i1}^{T},0,q_{i2}^{F}) = \frac{1}{2} - \epsilon$,
             \item If $\bar{x}_{2} \in \{l_{i1},~l_{i2},~l_{i3}\}$, then $T_{\mathcal{A}}(q_{i1}^{T},0,q_{i2}^{T}) = \frac{1}{2} - \epsilon$, and $T_{\mathcal{A}}(q_{i1}^{T},1,q_{i2}^{F}) = \frac{1}{2} - \epsilon$ 
             \item Otherwise, $\forall a \in \Sigma:~ T_{\mathcal{A}}(q_{i1}^{T},a,q_{i2}^{F}) = \frac{1}{2} - \epsilon$
         \end{itemize}         
     \end{itemize}
     \item For $ i^{*}_{C_{i}} \leq i < n$:
     \begin{itemize}
         \item $\forall a \in \Sigma:~T_{\mathcal{A}}(q_{i}^{T},a,q_{i+1}^{T}) = \frac{1}{2} - \epsilon$,
         \item If $x_{i} \in \{ l_{i1},l_{i2},l_{i3}\}$: 
         \begin{itemize}
             \item $T_{\mathcal{A}}(q_{i}^{F},1,q_{i+1}^{T}) = \frac{1}{2} - \epsilon$,
             \item $T_{\mathcal{A}}(q_{i}^{F},0,q_{i+1}^{F}) = \frac{1}{2} - \epsilon$,
         \end{itemize}
         \item Else if $\bar{x}_{i} \in \{ l_{i1},l_{i2},l_{i3}\}$:
         \begin{itemize}
             \item $T_{\mathcal{A}}(q_{i}^{F},0,q_{i+1}^{T}) = \frac{1}{2} - \epsilon$,
             \item $T_{\mathcal{A}}(q_{i}^{F},1,q_{i+1}^{F}) = \frac{1}{2} - \epsilon$,
         \end{itemize}
         - Else: $\forall S \in \{T,F\},~a \in \Sigma:~T_{\mathcal{A}}(q_{i}^{S},a,q_{i+1}^{S}) = \frac{1}{2} - \epsilon$
     \end{itemize}
\end{itemize}
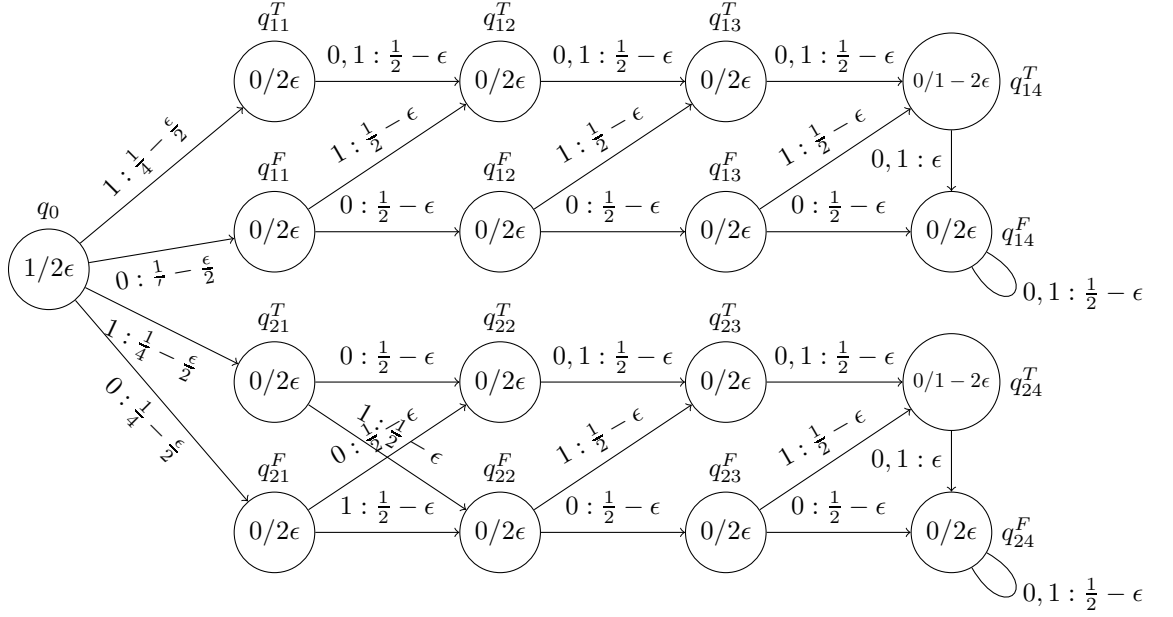
\begin{figure*}
\begin{center}
\begin{tikzpicture}
\tikzset{vertex/.style = {shape=circle,draw,minimum size=1.5em}}
\tikzset{edge/.style = {->,> = latex'}}
\node[vertex,label=above:$q_{0}$] (q0) at  (0,0) {$1/2 \epsilon$};
\node[vertex,label=above:$q_{11}^{T}$] (q11t) at  (3,2.5) {$0/2 \epsilon$};
\node[vertex,label=above:$q_{12}^{T}$] (q12t) at  (6,2.5) {$0/2 \epsilon$};
\node[vertex,label=above:$q_{13}^{T}$] (q13t) at  (9,2.5) {$0/2 \epsilon$};
\node[vertex,scale = 0.8,label=right:$q_{14}^{T}$] (q14t) at  (12,2.5) {$0/ 1-2 \epsilon$};
\node[vertex,label=above:$q_{11}^{F}$] (q11f) at  (3,0.5) {$0/2 \epsilon$};
\node[vertex,label=above:$q_{12}^{F}$] (q12f) at  (6,0.5) {$0/2 \epsilon$};
\node[vertex,label=above:$q_{13}^{F}$] (q13f) at  (9,0.5) {$0/2 \epsilon$};
\node[vertex,label=right:$q_{14}^{F}$] (q14f) at  (12,0.5) {$0/2 \epsilon$};
\node[vertex,label=above:$q_{21}^{T}$] (q21t) at  (3,-1.5) {$0/2 \epsilon$};
\node[vertex,label=above:$q_{22}^{T}$] (q22t) at  (6,-1.5) {$0/2 \epsilon$};
\node[vertex,label=above:$q_{23}^{T}$] (q23t) at  (9,-1.5) {$0/2 \epsilon$};
\node[vertex, scale = 0.8, label=right:$q_{24}^{T}$] (q24t) at  (12,-1.5) {$0/ 1-2 \epsilon$};
\node[vertex,label=above:$q_{21}^{F}$] (q21f) at (3,-3.5) {$0/2 \epsilon$};
\node[vertex,label=above:$q_{22}^{F}$] (q22f) at  (6,-3.5) {$0/2 \epsilon$};
\node[vertex,label=above:$q_{23}^{F}$] (q23f) at  (9,-3.5) {$0/2 \epsilon$};
\node[vertex,label=right:$q_{24}^{F}$] (q24f) at  (12,-3.5) {$0/ 2 \epsilon$};

\draw[->] (q0) -- (q11t) node[midway, above, rotate=45] {$1: \frac{1}{4} - \frac{\epsilon}{2}$};
\draw[->] (q0) -- (q11f)  node[midway, below,rotate=10] {$0: \frac{1}{'} - \frac{\epsilon}{2}$};
\draw[->] (q0) -- (q21t) node[midway, below, rotate=-25]{$1: \frac{1}{4} - \frac{\epsilon}{2}$} ;
\draw[->] (q0) -- (q21f) node[midway, below, rotate=-45]  {$0: \frac{1}{4} - \frac{\epsilon}{2}$};

\draw[->] (q11t) -- (q12t) node[midway,above] {$0,1: \frac{1}{2} - \epsilon$};
\draw[->] (q12t) -- (q13t) node[midway,above] {$0,1: \frac{1}{2} - \epsilon$};
\draw[->] (q13t) -- (q14t) node[midway,above] {$0,1: \frac{1}{2} - \epsilon$};
\draw[->] (q11f) -- (q12t) node[midway,above, rotate=30] {$1: \frac{1}{2} - \epsilon$};
\draw[->] (q11f) -- (q12f) node[midway,above] {$0: \frac{1}{2} - \epsilon$};
\draw[->] (q12f) -- (q13t) node[midway,above,rotate=30] {$1: \frac{1}{2} - \epsilon$};
\draw[->] (q12f) -- (q13f) node[midway,above] {$0: \frac{1}{2} - \epsilon$};
\draw[->] (q13f) -- (q14t) node[midway,above,rotate=30] {$1: \frac{1}{2} - \epsilon$};
\draw[->] (q13f) -- (q14f) node[midway,above] {$0: \frac{1}{2} - \epsilon$};
\draw[->] (q14t) -- (q14f) node[midway,left] {$0,1: \epsilon$};
\path (q14f) edge [out=330,in=300,looseness=8] node[right] {$0,1: \frac{1}{2} - \epsilon$} (q14f);

\draw[->] (q21t) -- (q22t) node[midway,above] {$0: \frac{1}{2} - \epsilon$};
\draw[->] (q21t) -- (q22f) node[midway,above,rotate=-30] {$1: \frac{1}{2} - \epsilon$};
\draw[->] (q22t) -- (q23t) node[midway,above] {$0,1: \frac{1}{2} - \epsilon$};
\draw[->] (q23t) -- (q24t) node[midway,above] {$0,1: \frac{1}{2} - \epsilon$};

\draw[->] (q21f) -- (q22f) node[midway,above] {$1: \frac{1}{2} - \epsilon$};
\draw[->] (q21f) -- (q22t)  node[midway,above,rotate=30] {$0: \frac{1}{2} - \epsilon$};
\draw[->] (q22f) -- (q23f) node[midway,above] {$0: \frac{1}{2} - \epsilon$};
\draw[->] (q22f) -- (q23t) node[midway,above,rotate=30] {$1: \frac{1}{2} - \epsilon$};
\draw[->] (q23f) -- (q24f) node[midway,above] {$0: \frac{1}{2} - \epsilon$};
\draw[->] (q23f) -- (q24t) node[midway,above,rotate=30]{$1: \frac{1}{2} - \epsilon$};
\draw[->] (q24t) -- (q24f) node[midway,left] {$0,1: \epsilon$};
\path (q24f) edge [out=330,in=300,looseness=8] node[right] {$0,1: \frac{1}{2} - \epsilon$} (q24f);
\end{tikzpicture}
\end{center}
\caption{A graphical representation of the PFA constructed from $F = (x_{1} \lor x_{2} \lor x_{3}) \wedge (\bar{x}_{2} \lor x_{3} \lor x_{4})$}
\end{figure*}
The construction above runs in $O(nk)$ time. \\
$\bullet$ \textbf{Construction of a RNN:} The RNN $R$ we'll construct is trivial, and it generates the quantitative language $R(w) = 2(\frac{1}{2} - \epsilon)^{|w|} \epsilon$. More formally, our RNN is defined as: 
\begin{itemize}
 \item N = 2 (2 hidden neurons),
 \item $\begin{pmatrix} h_{n_{1}}^{(0)} \\ h_{n_{2}}^{(0)} \end{pmatrix} = \begin{pmatrix} 0 \\ 0 \end{pmatrix}$ 
 \item Transition matrices: $W_{in} = \begin{pmatrix} 0 & 0 \\ 0 & 0 \end{pmatrix}$; $W_{0} = W_{1} = W_{\$} =  \begin{pmatrix} 0 \\ 0 \end{pmatrix}$
 \item Output matrices: $O = \begin{pmatrix} 0 & 0 \\ 0 & 0 \\ 0 & 0 \end{pmatrix}$, $O' = \begin{pmatrix} \log_{2}{\frac{1-2 \epsilon}{4 \epsilon}} \\ \log_{2}{\frac{1-2 \epsilon}{4 \epsilon}} \\ 0 \end{pmatrix}$ where $\log_{2}(.)$ is the logarithm to the base 2 
 \end{itemize} 
 What's left is to show that $R(w)$ = $2(\frac{1}{2}$ - $\epsilon)^{|w|}$ defines a consistent language model:
 \begin{proposition}
  For any $\epsilon < \frac{1}{2}$, the weighted language model defined as $f(w) = 2 (\frac{1}{2} - \epsilon)^{|w|} \epsilon$ is consistent.
 \end{proposition}
 \begin{proof}
 We have:
 \begin{equation}
 \begin{aligned}
     \sum\limits_{w \in \Sigma^{*}} f(w) &= 2\epsilon \sum\limits_{n \in \mathbb{N}} \sum\limits_{w: |w| = n} (\frac{1}{2} - \epsilon)^{n} \notag \\
     &= 2 \epsilon \sum\limits_{n \in \mathbb{N}} (1-2 \epsilon)^{n} \notag \\
 \end{aligned}
  \end{equation}
 By applying the equality: $\sum\limits_{n \in \mathbb{N}} x^{n} = \frac{1}{1 - x}$ for any $|x| < 1$ on the sum present in the right-hand term of the equation above, we obtain the result.
 \end{proof}

\begin{proposition}
Let $F$ be an arbitrary 3-SAT formula with $n$ variables and $k$ clauses. Let $\mathcal{A}$ be the PFA constructed from $F$ by the procedure detailed above, the probabilistic language generated by $\mathcal{A}$ is given as: 
\begin{equation*}
    \mathcal{A}(w) =  \begin{cases} 
         2 (\frac{1}{2} - \epsilon)^{|w|} \epsilon & if ~ |w| < n\\
          2(\frac{1}{2} - \epsilon)^{|w|} \epsilon [\frac{N_{w}}{k}\frac{1-2\epsilon}{2\epsilon} + \frac{k-N_{w}}{k} ] & if ~ |w| = n \\ 
         2(\frac{1}{2} - \epsilon)^{|w|} \epsilon [\frac{N_{w_{:n}}}{k} \frac{2\epsilon}{1 - 2\epsilon}+ \frac{k-N_{w_{:n}}}{k}] & else\\
      \end{cases}
\end{equation*}
\end{proposition}

\begin{proposition}
 For any rational number $\epsilon < \frac{1}{4}$, there exists a rational number $c_{\epsilon}$ such that $F$ is satisfiable if and only if $d_{\infty}(R, \mathcal{A}) > c_{\epsilon}$
\end{proposition}
\begin{proof}
 For any $w$ such that $|w| < n$, $|R(w) - \mathcal{A}(w)| = 0$ . \\
 For $|w| = n$, we have: 
 $$ |R(w) - \mathcal{A}(w)| = 2 \epsilon (\frac{1}{2} - \epsilon)^{n} \frac{N_{w}}{k} (\frac{1 - 4 \epsilon}{2 \epsilon}) $$
On the other hand, for $|w| > n$, we have: 
$$ |R(w) - \mathcal{A}(w)| =  2 \epsilon (\frac{1}{2} - \epsilon)^{|w|} \frac{N_{w}}{k} \frac{1 - 4 \epsilon}{1 - 2 \epsilon}$$
Note that we have for any $\epsilon < \frac{1}{4}$:
$$ \forall w \in \Sigma^{\geq n}:~~ |R(w) - \mathcal{A}(w)| \leq  |R(w_{:n}) - \mathcal{A}(w_{:n})| $$
 This means that, under this construction, the maximum is reached necessarily by a string whose length is exactly equal to $n$. Thus, we obtain:
$$d_{\infty}(R,\mathcal{A}) = 2\frac{\epsilon}{k} (\frac{1}{2} - \epsilon)^{n} \frac{1 - 4 \epsilon}{2 \epsilon} \max\limits_{w \in \Sigma^{n}} N_{w}$$

Note that $F$ is satisfiable if and only if $\max\limits_{w \in \Sigma^{n}} N_{w} = k$. As a result, pick any $s \in [k-1,k)$, and define $c_{\epsilon} =  2 \frac{\epsilon s}{k} (\frac{1}{2} - \epsilon)^{n} \frac{1 - 4 \epsilon}{2 \epsilon}$, the formula is satisfiable if and only if $d_{\infty}(R,\mathcal{A}) > c_{epsilon}$. 
\end{proof}

\begin{theorem}
 The Tchebychev distance approximation problem between consistent RNN-LMs and PFAs in finite support is NP-Hard.
\end{theorem}
\begin{proof}
 We reduce the 3-SAT satisfiability problem to our problem. Let $F$ be an arbitrary 3-SAT formula. Construct a PFA $\mathcal{A}$ and a RNN $R$ as specified previously. Choose a rational number $\epsilon < \frac{1}{4}$. Let $c_{\epsilon}>0$ be any rational number as specified in the proof of Proposition 5.4, and $N = n+1$. By Proposition 5.4, $F$ is satisfiable if and only if $d_{\infty}(R,\mathcal{A}) > c_{\epsilon}$, which completes the proof.
\end{proof}

\subsection{Conclusion}
 In this section, we gave a thorough treatment of connections between first-order RNN-LMs with $ReLu$ cells and different types of weighted automata from a computationnal viewpoint. The summary of obtained results are given as follows: \\ 
 (a) For general weighted first-order RNN-LMs with ReLu activation function: 1. The equivalence problem of a PDFA/PFA/WFA and a weighted first-order RNN-LM is undecidable; 2- As a corollary, any distance metric between languages generated by PDFA/PFA/WFA and that of a weighted RNN-LM is also undecidable; -The intersection between a DFA and the cut language of a weighted RNN-LM is undecidable; - The equivalence of a PDFA/PFA/WFA and weighted RNN-LM in a finite support is EXP-Hard; (b) For consistent first-order RNN-LMs with any computable activation function: - The Tcheybetchev distance approximation is decidable; - The Tcheybetchev distance approximation in a finite support is NP-Hard. A future direction of this work will be to enlarge the picture to include more complex RNN class of architectures, such as LSTMs/GRUs. \\
\section{Extractability of Finite State Machines from RNNs }
\label{chap:methods}
  
    $~~~$The extraction of a Finite State Machine from an RNN is defined as the whole algorithmic procedure used to transform an RNN into a finite state machine that is supposed to approximate well its behavior. Naturally, this procedure must involve, either directly or indirectly, a discretization of the RNN state space into partitions, each of which represents a state in the output finite state automaton. The question, then, is on selecting the best strategy for partitioning the state space. And, what kind of information an algorithm needs to get from the original network to obtain a good final partitioning of the RNN state space? Also, at least as important as the previous question, other questions arise: How to evaluate the performance of a given algorithm? Under what experimental setup? We'll see that some of these questions are quite challenging, and any experimental setup used to assess or compare the quality of FSM extraction algorithm will be intrinsically biased by the RNN training phase. \\ 
    We propose to divide this section into three parts: In the first part, we will give a brief literature overview of algorithms proposed in the literature to extract Finite State Machines from RNNs. Afterwards, we will analyze the experimental setup used to evaluate how well an extraction algorithm succeeds in the task of extraction, and the used evaluation metrics. The last part will be dedicated to present our experimental results, and analyze them through the lens of the generalization theory of Deep Learning. In light of this theory, a conjecture of why clustering-based algorithms work well will be suggested and theoretically motivated.   

\subsection{Brief overview of FSM algorithm extraction}
  As said earlier in the introduction of this part, FSM extraction algorithms need, at some point during their execution, to perform a partitioning of their clustering space to obtain the states of the output automaton. This partitioning could be explicit by extracting hidden state vectors of different prefixes, then use a clustering technique to obtain the final state automaton. An alternative way is to treat the network as a whole black box, and rely solely on its input-output in a form of some kind of an active learning protocol with a membership query oracle to discover \emph{"the approximate automaton"} that represents the network. Also, one can think of a hybrid combination of both techniques. In this section, we present two of these paradigms: The first, and the most widely addressed in the literature, which consists on clustering the hidden state space. And the second one that will be presented here will concern the use of the RNN as an \emph{gray box} oracle. These two families of algorithms represent the highly influencing and the most usually addressed in the literature. However, we need to mention that less well-known methods based on different strategies were proposed in the literature \cite{dong19}, \cite{koul18}.
  
\subsubsection{Clustering-based approaches for extracting finite state automata}
A major paradigm of extracting finite state machines from RNNs is based on the \emph{"clustering hypothesis"}. The main idea behind this assumption is that an infinite memory machine, RNNs falling in this case, trained to recognize a regular language tends to cluster the hidden state space into well-formed partitions that map to the canonical minimal DFA of the target regular language. Backed by this assumption, many algorithms aiming at extracting finite state machines from Recurrent Neural Networks were based on the use of some clustering/quantization technique of the RNN hidden state space to recover a finite structure. From a conceptual viewpoint, these algorithms can be seen as a special instanciation of the following \emph{meta-algorithmic} pattern: 
\begin{enumerate}
    \item \textbf{The clustering phase:} In this phase, hidden vectors of strings generated according to a chosen sampling strategy (Breadth-first search/Random sampling) are collected. Then, a clustering/partionning technique is employed to obtain the states of the extracted DFA,
    \item \textbf{The transition construction phase:} During this step, transition tables of each symbol in the alphabet are constructed to form transitions and corresponding labels of the DFA. To perform this task, the algorithm needs to traverse the RNN up to a given depth to fill in transition tables. Uusually, a non-determinism conflict occurs during this phase, where a symbol leads to two different states when starting from the same state. A common way to resolve this conflict consists at choosing the most frequent transition to form arcs between states.
    \item \textbf{Automaton minimization:} The resulting automaton of the previous step might be non-minimal, i.e. contains indistinguishable states. This step compresses this latter into its minimal equivalent version. This task is known to be performed in polynomial time.   
\end{enumerate}
 Algorithms following this \emph{meta-algorithmic} pattern differ mainly on the instanciation of two \emph{hyperparameters} : 
 \begin{itemize}
  \item The clustering/quantization technique used to form the states of the extracted DFA. This is fundamentally the most distinguishing feature of algorithms following this pattern. The rest of this section will be dedicated to discuss in depth different strategies explored in the literature, 
  \item The traversal strategy used to construct transition tables: two different traversal strategies are used to construct transition arcs of the DFA, Breadth-first search or random sampling. The choice of either strategy depends on the trade-off to be made between the computational cost and the risk of \emph{non-determinism conflicts}. A non-determinism conflict arises where two hidden vectors corresponding to two prefixes that belong to the same cluster will lead to two different clusters when the same symbol of the alphabet is applied, in which case the resulting automaton becomes non-deterministic. It's worth noting that only the random sampling strategy is concerned by this issue since in the case of breadth-first search strategy, the search is immediately pruned as we reach an already visited state.  When the random sampling strategy is used, this conflict is resolved statistically by counting for each cluster-symbol pair the number of transitions that lead to different clusters, and select the one with the maximum number. 
 \end{itemize}
  
  Our main focus in the rest of this section will be on the first aforementioned point, that is clustering/quantization based strategies used in the literature to convert the RNN hidden state space into a finite one. \\
  
 $\bullet$ \textbf{Quantization-based approaches} \\ 
 Discretizing the RNN hidden state space by quantization constitutes the oldest approach for recovering a finite structure from the infinite RNN memory \cite{Giles91}. It consists simply on partitioning the state space into equal partitions with a given resolution, where each partition represents a \emph{candidate} state of the final DFA. Then, a breadth-first traversal strategy is performed to promote reachable candidate states into validated ones, and build transition arcs of the final automaton. The setting of the resolution parameter is crucial for the success of such strategy. In fact, this latter will determine the trade-off between the size of the resulting automaton -the size of the final automaton could scale exponentially with the dimension of the hidden state space of the RNN, a problem known as \emph{a state explosion phenomenon}- and the accuracy of the resulting automaton. Intuitively, an excessively coarse partitioning will tend to cluster non-equivalent prefixes(in the sense of Myhill-Neyrode) into one single state. In practice, there is no known method to specify the resolution parameter, and It is done empirically by running the algorithm on different value parameters. Despite its simplicity, this algorithm succeeded to achieve good performance on vanilla RNNs with small-sized automata \cite{Giles91}, \cite{Giles92}. However, recent studies showed that the state explosion phenomenon is unavoidable when this algorithm is used in the context of modern RNN architectures, such as LSTMs/GRUs \cite{wang17}. \\ 
 
 $\bullet$ \textbf{Clustering-based approaches.} \\ 
 A major drawback of the previous method is the high sensitivity of the quality of extracted automata to the resolution parameter that needs to be selected \emph{a priori} before the run of the algorithm. To resolve this issue, a better alternative will be to give the extraction algorithm the capacity to adapt the coarseness of the state space after obtaining enough information on the arrangement of hidden state vectors in the state space. So, a better strategy will consist in clustering when enough states are visited through a BFS strategy. Then, a classical clustering algorithm is run to identify groups of close prefixes representing a candidate to be a Nerode equivalence class. Many clustering algorithms were proposed in the litterature, including K-means \cite{Zeng93}, Hierarchical Clustering \cite{Alquezar94}, SOM \cite{tino95} etc. Although K-means requires setting the number of clusters before the run of the algorithm, It's still the most popular one due to its simplicity. A general rule of using K-means for this task consists at giving a quite great number of clusters to the algorithm as a parameter, then use an automaton minimization procedure \emph{a posteriori} after obtaining the resulting automaton from the hidden vector space.    
  
 \subsubsection{FSM extraction as an active learning problem}
 Since the problem of automata extracting from an RNN assumes implicitly a complete access to the RNN machine we wish to approximate, then this problem can be perceived under an active learning setting. This is an approach taken by Weis et al. in \cite{Weiss18b}. The proposed approach attempts to adapt the well-know $L*$ algorithm \cite{Angluin87} used to infer DFAs. The membership query doesn't pose a serious problem as it requires simply the run of a forward operation on the network with the queried string as an input to obtain the result. On the other hand, the membership query oracle is more challenging. In fact, we have no guarantee that the target RNN, even if it generalizes well when trained to recognize a regular language, is indeed regular. To tackle this issue, authors proposed to create another automaton, called the abstraction automaton, to be representative of the RNN. This abstraction automaton maintains a partitioning of the state space, by making use of some clustering technique just as presented in the previous section. The algorithm then involves three actors: - The automaton to output, the abstraction automaton and the RNN. The equivalence query is then performed by testing if the automaton we try to learn is equivalent to the abstraction automaton. When the equivalence is negative, then there are two cases: Either the abstraction automaton or the learnt one are not consistent with the target RNN\footnote{This is only in the case of binary alphabet. If the alphabet size is greater than 2, then both the abstraction automaton and the learnt automaton could be wrong. This represents one of the weaknesses of this algorithm}. The algorithm resolves this conflict by querying the RNN on the output of the conflicting word. If the classification error on this word was made by the automaton to learn, then it's given back to it as a counterexample to update its observation table. If the error is made by the abstraction automaton, then it updates its partitioning of the hidden state space to account for this new prefix. Interestingly, unlike classical clustering strategies discussed in the previous section, the partitioning strategy is done in an incremental fashion by an adaptive learning strategy. This helps control the coarseness of the partition, then the time cost of its run.   
 
\subsubsection{Discussion on the experimental setup and evaluation measures}
In the literature, there is still a lack of a unifying experimental framework and consensual evaluation measures to assess and compare different DFA extraction algorithms. We give two reasons for that: 
\begin{itemize}
    \item Objectives of why we aim at extracting finite state machines from RNNs differ from one work to another. Depending on the objective of the experiment, the assessment of the quality of the extraction will be naturally different. In Section 3.2.3, we'll analyze in depth different measures of accuracy employed in the literature to evaluate the quality of extraction.
    \item Second, due to the \emph{two-stage} nature of the experiment as will be discussed later, there are many exogenous factors that will influence the outcome of experiments and might bias our interpretation of results. To cite some, the way we generate the training dataset, the initialization of the network, the training algorithm used to train the network. Some works explored in the literature focused mainly on the sensitivity of the quality of extracted DFAs to the initialization point \cite{Wang18a}.  As we shall see in next sections, the classical experimental setup is more sophisticated to be reduced to this sole factor, and many other experimental biases might intervene during the evaluation process.
\end{itemize}
We will present next a broad picture of the experimental setup classically used to test FSM extraction algorithms, and evaluation measures to assess their performances.   
\subsubsection{Paronamic view of the experimental setup}
In this section, we'll first describe the experimental setup that is classically used to evaluate algorithms aiming at extracting finite state machines from RNNs. We'll then follow up by an informal discussion on this experimental setup through the lens of classical results of computational learning theory concerning the learnability of DFAs. \\ 
From a panoramic point of view, the experimental setup comprises two steps:
\begin{itemize}[leftmargin=0cm]
    \item \textbf{Learning a regular language by an RNN:} This phase proceeds by fixing a RNN architecture of interest and train an RNN until we ensure that it generalizes well on the target regular language. This phase is performed as it's classically done for learning neural networks by generating a training dataset from an underlying selected distribution and train the network with the stochastic gradient descent algorithm or one of its variants. In most experiments encountered in the litterature, an RNN is considered to generalize well on the language if it achieves an accuracy rate over $99\%$ on a test dataset generated uniformly at random from a probability distribution pf strings with a fixed length. Note that, unlike in practical applications where neural networks are used in which the underlying distribution is unknown and the experimenter has only access to a sample assumed to be drawn independently from it, in the context of this experimental setup, the probability distribution represents a parameter that the experimenter can control, and a natural question that arises is at what extent this parameter influences the learnability of regular languages by RNNs. Besides the underlying data generating distribution, other parameters during this phase that might be suspected to influence the outcome is the algorithm used to learn the network, and the chosen loss function. In fact, those parameters introduce an implicit regularization bias on the learning process \cite{lei18}, \cite{ali20}, \cite{neyshabur17}, and may favor certain hypothesis in the search space over others. A question then arises is whether there exists a bias in the training procedure that favors hypothesis representing regular languages over others.  
    \item \textbf{Extracting a DFA from the RNN:} The second phase consists at running an algorithm that extracts a deterministic finite automaton from the trained RNN. This task can also be framed as a learning problem by taking the RNN as a query oracle. 
\end{itemize}
 Besides our main objective behind this experimental setup, this latter is very interesting to analyze it from a theoretical viewpoint, and may give insight into connections in terms of learnability between formal languages and RNNs. Conceptually, this experimental setup can be seen as a sort of Encoder-Decoder architecture where the learner in the first phase attempts to encode the target regular language into an intermediary neural representation then decodes it again, by means of the extractor algorithm, into a DFA. Intuitively, randomness is introduced in the process through the training sample used to learn the RNN. Hence, information is inevitably lost about the target language during the process. We informally address the following question:  \\

    \noindent \textbf{Question(Informal):} \textit{Does there exist a class of RNNs such that for any regular language, for any probability distribution $\mathbb{P}$ used to generate a training sample, this experimental setup runs in reasonable time and outputs a DFA that approximates well the target language with high probability.} \\

 This question, informally stated here, can be formally stated and hints to the polynomial-time PAC-learnability of the class of deterministic finite automata. \emph{Reasonable time} can be rigorously defined as being a polynomial time with respect to the size of the target grammar, the desired confidence and approximation error. Under cryptographic assumptions, the answer of this question is negative, otherwise the described experimental setup will represent an algorithm capable of inverting the RSA encryption function \cite{kearns94}. Another interesting consequence of the hardness result of PAC-Learning DFAs is that there could exist no RNN achitecture that holds the following three properties: (1) polynomial-time PAC Learnability, (2) expressiveness power at least equivalent to the class of regular languages (3) polynomial evaluability. \\ 
 In the context of our experimental objectives, the previous statement suggests the importance of the data generating distribution. In fact, given the informal statement above, the experimental setup taken as a whole doesn't hold the  distribution-free property for ensuring generalizability for any chosen probability distribution, and, due to its theoretical limitation, the bias related to the probability distribution chosen by the experimenter is inevitable.
\subsubsection{The data generation problem}
  The classical way of generating a training a dataset consists at sampling uniformly at random from a finite support whose size is fixed by the experimenter, then label samples that belong to the language as positive, and others as negative. The choice of the bound on the finite support depends on the size of the minimal automaton recognizing the language, and it's chosen to be at least twice the size of this latter. A theoretical justification of this choice is that an upper bound of the length of the the string that distinguishes between two minimal DFAs of size, say $k$, is known to be equal to $2k-1$ \cite{trakh73}. Thus, in principle, all the information about a given regular grammar of size $k$ is contained within this support. Also, when the target grammar is highly imbalanced, a quota sampling strategy used to keep a fairly balanced dataset. A classically employed learning strategy consists at fixing a ratio between positive and negative used the length of strings as a blocking variable. Consequently,  interestingly, in \cite{wang17}, authors proposed a totally opposite strategy where they suggested to up-sample the minority class and they gave theoretical argument based on the entropy of grammars to show the relevance of such choice. \\
 Besides the complexity of the grammar and the RNN architecture, we wish to evaluate in our experiments the influence of the generating distribution on the quality of extracting automata from RNNs. So, in addition to sampling strategies described above, we propose to test another sampling strategy that stems its legitimacy from the notion of Myhill-Neyrode equivalence discussed in earlier sections. The basic intuition is that an RNN trained to recognize a regular language must be trained to preserve the equivalence relation between prefixes that belong to the same state in the canonical minimal DFA. An RNN that is capable to discriminate between prefixes that belong to different states will lead to good generalization capabilities. Our general procedure for forming a training dataset runs as follows: 
 \begin{itemize}
     \item Run a breadth-first search to generate $k$ prefixes for each state of the target canonical automaton, where $k$ will be set according to its size,
    \item  Sample uniformly at random strings of length proportional to the size of the automaton, 
    \item Concatenate prefixes obtained in the first step with random strings obtained in the second step and label the obtained dataset,
    \item Ensure that the ratio of difference of labels between every pair of equivalence prefix classes is approximately 1:1,
 \end{itemize}
 
\subsubsection{Evaluation measures}
The problem of evaluating the quality of extraction will naturally depend on the objective of the experiment. In the literature, the problem of extracting automata from recurrent neural networks is motivated by two different objectives: The problem of understanding how languages are internally represented within different classes of RNN architectures, and the problem of providing a good symbolic approximation of a target automaton. The former problem holds a theoretical interest and serves to understand how RNNs arrange states within their continuous state space when learned to recognize a regular language. The latter problem is more practically oriented, and finds mainly its applications in the problem of Model Checking and formal specification. \\ 
The choice of how to evaluate the quality of extraction will naturally depend on the objective of the experiment. We'll present in the following commonly used evaluation measures to assess the quality of a DFA extraction algorithm: 
\begin{itemize}
    \item The generalization capacity of the extracted automaton with respect to the trained RNN: This measure is used when the goal of the experiment is to evaluate the quality of how well the extracted DFA approximates the behavior of the target RNN. The generalization capacity is evaluated by using a test dataset generated uniformly at random from different sizes of finite support. An important property of a good algorithm is to evaluate its generalization capacity in a support larger than the one used for extraction. 
    \item The success rate of extracting the DFA: This measure is more conservative and It evaluates if the obtained automaton is equivalent to the target automaton. In order to evaluate robustly this measure, we need to isolate the noise injected during the training phase of the RNN. For this, the training and extraction protocol is run multiple times and the average number of equivalent extracted automata is measured. Intuitively, it's legitimate to believe that the size of the automaton is going to be proportional with performances in terms of success rates. We'll see in our experiments that this is not systematically the case, and that when dealing with learnability issues, other statistical factors are involved to account for the complexity of learning automata. 
    \item Computational efficiency: it is measured by the running time for the DFA extraction algorithm
    \item The size of the extracted DFA: if the generalization capacity in a finite support is the only factor that accounts for the performance of a DFA algorithm, then one can simply run a \emph{state-merging} algorithm, such as RPNI \cite{higu10}, under this support and a perfect matching of the support is obtained. For fairness, we need to account for the size of the extracted DFA to compare between DFA extraction algorithms. Note that this measure is also correlated with the computational efficiency. 
\end{itemize}
\subsection{Experiments and results}
\subsubsection{Details of the experiment}
Before presenting details of the experiment, let's first recall its main objectives. First, we aim at comparing different algorithms that are present in the literature in terms of their computational efficiency, the quality of extracted DFAs by these algorithms according to two aspects: The approximation quality of the extracted automaton, or even, more rigorously, the success rate of recovering the target automaton and the second aspect is the size of the extracted DFA. A second objective is to evaluate classes of RNN architectures from which extracting DFAs is more easier. Next, we'll present details of RNN classes and selected architectures used for the sake of our experiments, and the used training setup. Details about the benchmark will also be given. \\

$\bullet$\textbf{Benchmark and Data set generation.} \\ 
The used benchmark is the one traditionally used to examine DFA extraction algorithms, Tomita Grammars \cite{Tomita82}. The description of these tomita grammars are given in Table 3.1. The DFA with maximal size among Tomita grammars is equal to 5 (Grammars 3 and 7).  
\begin{table}
\begin{center}
 \begin{tabular}{|c | c | c| c |} 
 \hline
 Grammar & Description & Length & Train \\ [0.5ex] 
 \hline\hline
 1 & $1*$ & 22 & 800 \\ 
 \hline
 2 & $(10)*$ & 22 & 800 \\
 \hline
 3 & Strings with no $0^{2n+1}1^{2m+1}$ as substring & 25 & 1500  \\
 \hline
 4 & Strings not contain $000$ as substring & 22 & 950  \\
 \hline
 5 & strings with even number of 0/1's & 25 & 1200  \\ 
 \hline
 6 & Difference between number of 0/1 is $3n$ & 30 & 1200  \\ 
 \hline 
 7 & $0^{*}1^{*}0^{*}1^{*}$ & 35 & 1750  \\
 \hline 
\end{tabular}
\caption{Description of Tomita Grammars and the maximum support length for data set generation}
\end{center}
\end{table}

For each Tomita grammar using two training datasets, one based on a classical uniform sampling with up-sampling strategy for grammars that have imbalanced, and the other sampling strategy using \emph{prefix quota-based} strategy as explained in earlier section. Validation and test datasets are generated uniformly at random from a support of fixed length. Details of the length of the support and the size of the training dataset are given in Table 3.1. \\

$\bullet$\textbf{Selected architectures and the training setup.} \\ In our experiments, we choose to compare between four classes of RNN architectures: First-order RNNs and second-order RNNs with sigmoid and \emph{tanh} functions, LSTMs and GRUs. For each class of RNNs, we select three architectures with hidden state dimension: 50, 100, 150. It is worth noting that RNNs with Relu activation functions is exluded in our experiments because clustering-based methods are sensitive to the compactness of the support of the hidden state state space. RNN with Relu activation function have an infinite support, as its image space is the whole positive orthant.

 In order to control the bias introduced by the training phase to the extraction phase, we selected only networks that has achieved a rate $>99\%$ of accuracy on the training loss and $>85\%$ on the test dataset. We note that to achieve those rates, the training had to be performed multiple times for first-order RNNs with small state size dimension. LSTMs and GRUs had no difficulty to fit the desired accuracy rates.  Networks were initialized randomly using a gaussian distribution, and Adam was used as an optimizer. We justify our first choice by the fact that we want our experimental setup to run in conditions similar to the usual way RNNs are initialized. However, this factor represents another random parameter in the experimental setup that might influence the outcome, especially the success rate of extracted DFA. For this, we trained every architecture five times, and stored each of the obtained network to be used for extraction. 
\subsubsection{Experimental results}
We compare the result of three algorithms: Clustering-based algorithm using K-means as a clustering method with a BFS search strategy, the adaptation of the $L*$ algorithm as proposed by Weiss et al \cite{Weiss18b}, and a quantization-based algorithm \cite{Giles91} as a baseline. Results of the experiment are given in the appendix. We tested $K=5,10,15$ as a parameter for the number of clusters for the clustering based method. As for the $L*$ algorithm, parameters used are similar to those given by authors in their original article. We used the source code provided by authors\footnote{\url{https://github.com/tech-srl/lstar_extraction}} for implementation. Results of the experiment for grammars 1, 4 and 5 are given in the appendix. 

\subsubsection{Interpretation of results}
 The results obtained show that the $L*$ algorithm clearly outperforms all other algorithms in terms of time efficiency, exact extractability and the generalization accuracy. The quantization-based algorithm provides poor performance, and when applied to architectures with large hidden state space spends an excessively high running time to succeed. \\ 
 In the following, we'll discuss two observed phenomena that we observed in our experiments: The high average success rate of some algorithms, and the generalizability of the extracted automaton to uncovered support. \\
 
\noindent $\bullet$ \textbf{Connection between the learnability of a language and its entropy.} \\
 From an information-theoretic viewpoint, the fact that some methods enjoys high average success rates is rather surprising. It means that they were capable of correcting the noise injected to the language during the training phase without the assistance of any additional information that may compensate the loss  except the one given implicitly by restricting our description to belong to the class of regular grammars.  One may account this fact to the relatively small size of grammars used in the experiment. However, more surprisingly, grammar 7 whose DFA size is equal to 5 enjoys higher accuracy rates than some smaller sized grammars. Our experiments seem to confirm theoretical insights proposed by Zhang et al. \cite{Zhang18} in which they connect the successful learnability of regular languages by RNNs with the entropy of the language, i.e. lower entropy languages(or, alternatively, those that holds a small information content) are easier to learn than higher ones. In traditional grammatical inference field, the complexity of learning grammars is related to their size. However, what this observation seems to suggest is that another hierarchisation of formal grammars in terms of their information-theoretic properties would be more relevant when dealing with their learnability. This information-theoretic perspective of learning formal languages through RNNs (that serve the role, in this case, of a noisy channel) is left as a future direction of research. \\
 
\noindent $\bullet$ \textbf{Understanding the clustering hypothesis through Lipchitzianity and connections with Deep Learning Generalization Theory.} \\
 Another observation that we report in light of our experiments is that clustering-based algorithms enjoy two empirical properties: Generalize well beyond their covered support during the run of the algorithm. This is witnessed in all clustering algorithms used in our conducted experiments, which seems to confirm that the resulting network is stable, and that the clustering hypothesis is valid. Authors in \cite{dhingra17} hinted that this fact is due to the smoothness property that trained network enjoy in practice. Interestingly, this smoothness property is intimately related to the spectral norm of weight matrices \cite{barlett17}, \cite{Bietti18}, a notion suspected by Deep Learning theorists to account for the generalization capabilities of Deep Learning architectures. Recently, authors in [] gave a generalization bound of RNN-Relus which depend on the $l_{1}$-norm of its transition matrix. Informally, if It's proven that It's possible to extract \emph{fairly small} finite  state automata can compress with a \emph{fairly good} precision RNNs with small norm complexity, as experiments seem to suggest, then this would mean that the real effective capacity of architectures with low spectral normal complexity can be quantified by the \emph{fairly small} finite state automata by which It is well-approximated.  The idea of proving generalization bounds on Deep Learning Architectures by means of compression was formalized by Arora et al. in \cite{arora18}. In light of our experimental results, we conjecture the following: 
 \begin{conjecture}
  Finite state automata tend to approximate well RNNs with low spectral norm complexity.
 \end{conjecture}
 We give here some theoretical insights in the perspective of making this intuition more formal: \\
   $~~~$ The smoothness of a function is formalized by the notion of Lipchitzness. Let $L \in \mathbb{R}^{+}$. A function $f: \mathbb{R}^{d} \leftarrow \mathbb{R}^{d'}$ is said to be $L$-Lipchitz with respect to norms $<||.||_{d},~||.||_{d'}>$, where $||.||_{d}$ (resp. $||.||_{d'}$) are arbitrary norms in $\mathbb{R}^{d}$ (resp. $\mathbb{R}^{d'}$) if It satisfies the property: $||f(x) - f(y)||_{d'} \leq L ||x - y||_{d}$. An interpretation of this property is that if two vectors $x,~y$ are close to each other in the vector space, then the difference of their image after application of the function $f$ will not go further from each other as their gap is controlled by the distance between $x,~y$. This is exactly the desired property in the context of learning regular languages and which might explain the ease by which the hidden state space of the network can be clustered: In other words, two hidden vectors representing two prefix and belonging to the same cluster (the same state) need to stay close to each other to be grouped in the same cluster when a new symbol is run from these two vectors. Let $f_{W}$ be a transition function of some RNN architecture. Its Lipchitzianity will naturally depend on the transition matrix. For illustration, we give the following proposition for the case of ReLu functions.
   \begin{proposition}
     For any $d>0$, any matrix $W \in \mathbb{R}^{d \times d}$, the function $f(x) = ReLu(Wx+b)$ is $||W||_{F}$-Lipchitzian, \\
     where $||W||_{F} = \sqrt{\sum_{ij}W_{ij}^{2}}$ is the Frobenius norm, and $||.||_{2}$ is the standard euclidian norm,
   \end{proposition} 
  
   \begin{proof}
     Observing that $ReLu$ is a 1-Lipchitzian function, we have: 
     \begin{equation*}
         \begin{aligned}
           ||f(x) - f(y)|| &= ||ReLu(Wx+b) - ReLu(Wy+b)|| \\ 
           & \leq || Wx-Wy||  \\
           & \leq ||W||_{F} ||x-y|| \\
         \end{aligned}
     \end{equation*}
    The first inequality is due to the Lipchitzianity property of the $ReLu$ function, and the second is due to the inequality: $||Ax||_{2} \leq ||A||_{F} ||x||_{2}$.
   \end{proof}
  It's worth mentioning that the same proposition can be stated for $tanh$ and sigmoid functions. \\
  
\noindent  $\bullet$ \textbf{The contractive network and the convergence to the single state automaton:} An insightful remark is that if the transition $f$ is contractive (i.e. $L < 1$), then any two vectors in the vector space will get closer to each other at an exponential rate. From the perspective of automata extraction, this means that all prefixes will be Nerode-equivalent under a finite horizon in the network. This additional insight shows the importance of the analysis of the spectral norm of the networks parameters to develop better strategies of clustering.  \\ 
  This smoothness property might also explain why adding intermediate layers in the architecture, as suggested by different authors, can help the learning phase. Those intermediate layers would enhance the smoothness properties of the transition function.
  
  \subsection{Conclusion}
  In this section, we addressed the problem of extracting Finite State Machines from RNNs. We discussed issues related to the experimental setup. In light of our experiments and new findings of generalization deep learning theory, we provided some insights into the longstanding \emph{"Clustering theory"}. \\ 
  The major limitation of all algorithms presented in this section is that they rely on a finite support traversed by a BFS strategy to recover a final state automaton. In the next section, we'll propose another paradigm of extraction. This novel paradigm will take the form of an extension of an active learning framework where RNNs and FSMs are given a language model of reference which provides a ground truth of the approximation quality. 
\section{Learnability of RNNs by Finite State Machines with a reference language model oracle}
We discussed in earlier sections how objectives differ when one considers the problem of extracting finite state machines from RNNs: Either the theoretical motivation of understanding how RNNs represent languages, or a practical necessity of \emph{compiling} the neural network into a transparent, computationally efficient machine deployable in real-word applications where requirements in terms of interpretability and the computational performance in terms of string processing are of utmost importance. In this latter case, RNNs are trained to perform a specific task, and are required to give accurate predictions when presented with unseen examples. The performance of RNNs can't be decoupled from the task for which it was designed and trained. Transitively, the quality of how well a machine approximates well a recurrent neural network is necessarily bound to its predictive power on the task for which the original RNN was designed. In simpler terms, a DFA that claims to approximate well the behavior of an RNN is required to provide evidence of its capacity to generalize well to new unseen examples with performances fairly close to the original RNN. Besides their theoretical interest, methods presented in previous sections dismissed this practical aspect of approximating RNNs, and extracted automata are totally agnostic to the target task, giving equal importance to all strings in the covered finite support during their execution, to not mention its scalability issues and a lack of theoretical guarantees. The fact is that generally, in practice, not all strings are of equal importance: some strings are more likely to occur than others. To illustrate, let's take the example of sentiment analysis: suppose we have an RNN trained as a classifier to predict whether a given sentence communicates a positive or a negative sentiment. And, our aim is at extracting a finite state machine, say DFA, that approximates well the behavior of the RNN. A typical behavior of algorithms presented so far would lead to a resulting automaton that assigns equal importance to a sentence like \emph{"I am happy"} and \emph{"Be to tree"}. Trivially, this behavior is undesirable in this practical scenario. First, due to their scalability limitations, those algorithms would fail to give a significant automaton in a reasonable time. And, even if one of these algorithms was given enough computational and memory resources over a sufficiently large support, It would be too large to be exploited in real-world applications. Another important aspect that needs to be questioned is even the notion of approximating a non-regular language by a regular automaton. Indeed, RNNs, even when trained to recognize regular languages, are not guaranteed to be regular, and attempting to measure the quality of an extracted automaton with respect to its capacity to recognize a non-regular target language is questionable. In light of this previous discussion, we argue that any successful strategy of finite state machine extraction requires the assistance of an \emph{oracle} in the form of a language model that guides the extraction algorithm during its run. This language model also will serve as a \emph{referee} between the teacher \emph{(the RNN)} and the learner \emph{(the automaton)}, and decides whether this latter is a good approximation of the target RNN. This reference language would take the form of any type of probabilistic machine: PDFA/PFA/RNN-LMs etc. Under this setting, the problem of extracting finite state machines is framed as a learning problem.  \\
 In this section, we'll propose an extension of the active learning framework that best suits the practical objective of approximating RNNs with FSMs. The framework is designed to be realistic in terms of the feasibility, the oracle query complexity and doesn't enforce the regularity requirement of the RNN.  
 
 \subsection{The learning framework with a reference language model oracle}
  In the classical active learning framework, the learner is given access to an oracle teacher which provides a privileged information about the target function (in our case, it's the target RNN). For the rest of this section, it will be denoted by $R$). In our learning scenario, an additional actor is involved in the learning setup: A language model of reference $\mathbb{P}$ over the set of strings. This language model would be accessed by three types of oracles: 
  \begin{itemize}[leftmargin=0cm]
      \item \textbf{The sampling oracle:} an oracle that samples randomly a string $x$ according to $\mathbb{P}$  and returns the sampled string to the learner. This is the default oracle that is used in the classical active learning framework, 
      \item \textbf{The probability weight oracle:} an oracle that samples randomly a string $x$ according to $\mathbb{P}$ and returns the pair $(x, \mathbb{P}(x))$ to the learner, 
      \item \textbf{The highest probable string oracle:} An oracle that takes the history set of already sampled words by the learner denoted $H$, and returns a string $x \in argmax_{w \in \Sigma^{*} \bigcap \bar{H}} ~~\mathbb{P}(w)$. Informally, this oracle delivers the most probable string that the learner hasn't seen yet.   
      \item \textbf{The test oracle:} an oracle that receives an automaton $\mathcal{A}_{t}$ from the learner and a parameter $\epsilon \in (0,1)$, and returns yes if $\mathbb{P}(R(x) \neq \mathcal{A}_{t}(x)) \leq \inf\limits_{{\mathcal{A} \in \mathcal{DFA}}(\Sigma)} \mathbb{P}(R(x) \neq \mathcal{A}(x))  + \epsilon$, no otherwise.  
 \end{itemize}
 To ease notation, we'll denote in the rest of this section $\mathcal{L}_{\mathbb{P}}(\mathcal{A}) = \mathbb{P}(R(x) \neq \mathcal{A}_{t}(x))$, and $\mathcal{L}_{S}(\mathcal{A}$ to be the ratio of misclassified examples in $S$ by the automaton $\mathcal{A}$.   
 Notice that the test made by the test oracle assumes implicitly that $R$ doesn't necessarily have to belong to the class of regular languages, If It's the case, we'll have $\min\limits_{\mathcal{A} \in \mathcal{DFA}(\Sigma)} \mathbb{P}(R(x) \neq \mathcal{A}(x)) = 0$.  
 Moreover, the learner has access to the membership query oracle which returns the label of a string query $x$. \\ 
 Under the active learning framework, there are many possible learning settings: 
 \begin{itemize}
     \item \textbf{Batch Learning:} Under this setting, the learner queries the oracle only once for training, where It is given data in batch. Then, the learner exploits this batch data to infer the automaton, 
     \item \textbf{Incremental Learning:} The learning configuration runs in rounds. At round $t$, the learner queries oracles accessible to it, then based on the given answers, updates its hypothesis automaton selected in the previous round, 
 \end{itemize}
The efficiency of an algorithm will be measured in terms of the number of oracle calls it has to make to reach PAC-style convergence guarantees. \\ 
The learning framework is rich with many alternative learning scenarios. In the next section, we'll provide theoretical analysis of two scenarios: the case of batch learning with Membership, and the sampling Oracle, and the case of batch learning with the highest probable string oracle and the membership oracle. We'll leave analysis of other scenarios for future research. 

\subsection{Batch learning with the membership and sampling oracle}
 In this section, we'll propose a learning scenario where the learner has only access to the membership and the sampling oracle. It is important to mention at this point that there are two ways the membership oracle could be accessed: either passively, i.e. the learner queries the sampling oracle, then immediately queries its label from the membership oracle, or actively by decoupling the sampling oracle from a systematic call to the membership oracle, and the learner can query the membership oracle independently from the sampling oracle. Our framework encompasses both scenarios. The former case resembles more to the classical supervised learning framework, and It's known to be the hardest one in terms of PAC learnability guarantees \cite{kearns94}. The latter gives more favorable learning conditions to the learner and it is known that the class of deterministic finite state automata are PAC learnable with membership queries. The proposed work in this section addresses the more challenging case: The passive learning case. It is known that, under cryptographic assumptions, the class of DFAs are known to be not PAC-Learnable with a polynomial sample complexity \cite{kearns94}. Under our agnostic setting, the problem is even more challenging as the target concept doesn't necessarily have to belong to the concept class of finite state automata. Hence, deriving distribution-free generalization bounds in this learning scenario is theoretically impossible. When faced with PAC Hardness results, there are generally four types types of strategies to tackle this issue: (1) Constraining the space of probability distributions under which the learning becomes successful (which translates from a theoretical point of view to distribution-specific generalization bound, thus scarifying the powerful distribution-free property), (2) or by constraining the expressiveness power of your concept class, (3) or by assuming a sort of \emph{topology} of your concept class that makes learning easier (Nonuniform learnability paradigm is a successful example of such strategy \cite{benedeck94}), (4) Assisting the learner with oracles during training that provide it with privileged information (Exact learning framework is another successful example of such strategy. \\ 
 We present in the rest of this section the most basic learning setup, yet the most challenging, where the learner has access to the sampling oracle, obtains an unlabelled batch data drawn independently at random from a distribution $\mathbb{P}$, then queries the membership oracle to label them. We'll call such learner a \emph{passive} learner.
 In the rest of this section, we'll first present the learning protocol. Then, we'll describe assumptions made on the probability distribution and the target function which give a sufficient condition to recover uniform PAC bounds from a non-uniform PAC one. Finally, based on those assumptions, we shall give bounds on the number of queries required to recover, \textit{with high probability}, an approximate automaton of the RNN. We recall the meaning of big $O$ notation: $f(x) \in O(g(x))$ if and only if there exists a positive number $C$, and a number $x_{0}$, such that for all $x > x_{0}:~f(x) \leq Cg(x)$.
 \subsubsection{The learning protocol}
 The learning protocol of the passive learner is rather simple and runs as follows:
 \begin{enumerate}
     \item The learner makes $m$ queries to the sampler oracle,
     \item After getting $m$ samples, the learner makes $m$ queries to the membership oracle to obtain labels, 
     \item  After obtaining a labeled sample, denoted by $S$, the learner minimizes the following optimization problem, and outputs the resulting DFA:
     $$L(\mathcal{A}) = \mathcal{L}_{S}(\mathcal{A}) + C\sqrt{|\Sigma||\mathcal{A}|(\log(|\mathcal{A}|) + 1) }$$
     where $C > 0$ is a universal constant, 
 \end{enumerate}
  $~~$ \textbf{Note.} \emph{In this work, our main attention is focused on the feasability of learnability as defined by the sample size required by a learner to learn the target concept. If we analyze the learning setup from a strictly computational viewpoint, then the legitimate question related to the complexity of the optimization problem will be raised. } \\ 
  We give here the main assumption of this work: 
 \begin{assumption}
  Let $\mathbb{P}$ be a distribution probability over strings, and $R$ be a target concept. Define the function (that depends implicitly on $\mathbb{P}$, and $R$) as follows :
  \begin{equation*}
      \begin{aligned}
        \zeta_{\mathbb{P},R}: (0,1)  &\rightarrow \mathbb{N} \\
             \epsilon & \rightarrow \min\{|\mathcal{A}|:~~\inf\limits_{\mathcal{A} \in \mathcal{DFA}(\Sigma)}{\mathcal{L}_{\mathbb{P}}(\mathcal{A}) \leq \mathcal{L}_{\mathbb{P}}(\mathcal{A}) <  \inf\limits_{\mathcal{A} \in \mathcal{DFA}(\Sigma)}{\mathcal{L}_{\mathbb{P}}(\mathcal{A}) + \epsilon}}\}
      \end{aligned}
  \end{equation*}
   We have: 
   \begin{equation}
       \zeta_{\mathbb{P},R}(\epsilon)~~ \underset{\epsilon \to 0^{+}}{=}  ~~ O(\frac{1}{\epsilon})
   \end{equation}
 \end{assumption}
 
 To ease notation, we will write only $\zeta$, when $\mathbb{P}$ and $R$ are clear from context. Before proving that this function is well-defined, we propose first to give an interpretation of this assumption. The function $\zeta$ controls the growth rate of automata as we try to decrease the approximation error, i.e. the size of automata that approximates well the target function will have a size that scales linearly with the desired approximation error.  We note that this assumption is a property that is shared between a target function and a probability distribution, and not intrinsic to one of them. We also need to mention that PAC learnability results with a polynomial time of oracle calls that we will show in the rest of this section holds even if we the growth rate of $\zeta$ is polynomially agressive (i.e. $\zeta(\epsilon) \underset{\epsilon \to 0+}{=} O(poly(\frac{1}{\epsilon})$) where $poly(.)$ is a polynomial). Only the convergence rate of the number of oracle calls will be affected in a negative sense, but still remains polynomial. We first show that the function $\zeta$ is well-defined for any target function $R$, and any probability distribution $\mathbb{P}$. 
 
 \begin{proposition} 
  Fix an alphabet $\Sigma$. For any $R: \Sigma^{*} \rightarrow \{0,1\}$, for any probability distribution over $\mathbb{P}$, the function $\zeta$ is well-defined
 \end{proposition}
 \begin{proof}
    Fix any $R: \Sigma^{*} \rightarrow \{0,1\}$, fix any probability distribution over $\mathbb{P}$. We shall prove that $\zeta(\epsilon)$ is well-defined $\forall \epsilon \in (0,1)$, i.e. that the set present in the formula of the function is not empty. We proceed by analyzing two cases: Case 1: There exists an automaton $\mathcal{A}$ that satisfies $\inf\limits_{\mathcal{DFA}(\Sigma)} \mathcal{L}_{\mathbb{P}}(\mathcal{A})= 0$. Then, the set is not empty since this automaton always exists in this set. Case 2: There is no automaton that satisfies the property. We prove by contradiction: Suppose no automaton exists within the specified interval. Then, for any automaton $\mathcal{A}$, $\inf\limits_{\mathcal{A} \in \mathcal{DFA}(\Sigma)} \mathcal{L}_{\mathbb{P}}(\mathcal{A}) + \epsilon \leq \mathcal{L}_{P}(\mathcal{A})$, thus $\inf\limits_{\mathcal{DFA}(\Sigma)} \mathcal{L}_{\mathbb{P}}(\mathcal{A})$ is not an infinimum since there exists a greater quantity that minimizes $\mathcal{L}(\mathcal{A})$ for all $\mathcal{A} \in \mathcal{DFA}(\Sigma)$, which leads to a contradiction. 
 \end{proof}
In order to ease notation for the rest of the proof, we will define the following set of automata: 
 $$ \mathcal{A}_{min}(\epsilon) = \{|\mathcal{A}|:~~\inf\limits_{\mathcal{A} \in \mathcal{DFA}(\Sigma)}{\mathcal{L}_{\mathbb{P}}(\mathcal{A}) \leq \mathcal{L}_{\mathbb{P}}(\mathcal{A}) <  \inf\limits_{\mathcal{A} \in \mathcal{DFA}(\Sigma)}{\mathcal{L}_{\mathbb{P}}(\mathcal{A}) + \epsilon}}\}$$
 
 The proof of the theorem is similar to how SRM generalization bounds are generally derived. The first step consists at deriving a generalization bound on the family of finite state automata of size $n$. The basic tool to prove this is based on the VC-dimension of $\mathcal{DFA}_{n}^{\Sigma}$ (See \cite{vapnik95} for a full treatment of the VC-dimension of a concept class). This is a finite set, and its VC-dimension is upper bounded logarithmically with the size of this set. Then, we use a union bound over all these spaces, each of which is assigned a weight that represents its importance under a model selection setting. Then, based on our main assumption on the target function $R$, and the probability distribution $\mathbb{P}$, we shall be able to recover a uniform bound on the sample size required to obtain PAC-like guarantees of the learning setup. Let's first recall few technical lemma: 
\begin{lemma}{\cite{shalev14}} 
 Let $\mathcal{X}$ be an instance space, and $\mathcal{H}$ is a set of concepts in $\mathcal{X}$ with $d_{vc}(\mathcal{H}) < \infty$\footnote{$d_{vc}$ refers to the notion of Vapnik-Chervonenkis dimension, a full description and definition of this notion can be found in \cite{vapnik95}}. Then, for any target concept $h \in \mathcal{H}$, for any probability distribution $\mathbb{P}$ on $\mathcal{X}$ , for any $\epsilon, \delta$ belonging to  $(0,1)$, for any $m \in \mathbb{N}^{+}$. we have, with a probability at least $1-\delta$ on a dataset $S$ of size $m$ generated uniformly at random from $\mathbb{P}$, the following event: 
 \begin{equation*}
     \sup\limits_{h \in \mathcal{H}} | \mathcal{L}_{P}(h) - \mathcal{L}_{S}(h)| \leq O(\sqrt{\frac{d_{vc} + \log(\frac{1}{\delta})}{m})}
 \end{equation*}
\end{lemma}

\begin{lemma}{(VC-dimension of $\mathcal{DFA}_{n}(\Sigma)$) \cite{ishigami97}}
 For any $n \in \mathbb{N}$,  for any alphabet $\Sigma$ with $|\Sigma| \geq 2$, we have:
 \begin{equation}
     d_{VC}(\mathcal{DFA}_{n}(\Sigma)) \leq O(|\Sigma|nlog(n))
 \end{equation}
\end{lemma}

\begin{theorem}
There exists an universal constant $C > 0$, such for for any finite alphabet $\Sigma$ with size greater or equal to $2$, for any distribution $\mathbb{P}$ on $\Sigma^{*}$, for any $\delta \in (0,1)$, we have with probability at least $1-\delta$ on a sample of size $m$ drawn independently from $\mathbb{P}$ the following 
\begin{equation*}
    \forall \mathcal{A} \in \mathcal{DFA}(\Sigma):~~ |\mathcal{L}_{\mathbb{P}}(\mathcal{A}) - \mathcal{L}_{S}(\mathcal{A})| \leq C\sqrt{\frac{|\Sigma||\mathcal{A}|(\log(|\mathcal{A}|) + 1)  +  \log(\frac{1}{\delta})}{m}}
\end{equation*}
\end{theorem} 
\begin{proof}
The proof follows the classical machinery for proving Structural Minimization Bounds. The idea is to partition the space of finite state automata as: $\mathcal{DFA}(\Sigma) = \bigcup\limits_{n=1}^{\infty} \mathcal{DFA}_{n}(\Sigma)$ and assign a weight to each partition by a weighting function: $w: \mathbb{N}^{+} \rightarrow (0,1)$ such that $\sum\limits_{n \in \mathbb{N}^{+}} w(n) \leq 1$. We choose $w(n) = 2^{-n}$. \\
Fix an alphabet $\Sigma$ where $|\Sigma| \geq 2$. Define the quantity: $$\epsilon(m,n,\delta) = C\sqrt{\frac{|\Sigma| n \log(n) + \log(\frac{1}{\delta})}{m}} $$ where $C>0$ is a constant that will be defined later. \\
Let's define the following events: \\
$~~~E = \{S:~ \exists A \in \mathcal{DFA}(\Sigma)~~~ |\mathcal{L_{\mathbb{P}}(\mathcal{A})} - \mathcal{L_{S}(\mathcal{A})}| \geq \epsilon(m,|\mathcal{A}|, w(|\mathcal{A}|)\delta \}$ \\
$ ~~~~E_{n} = \{ S: ~~\exists \mathcal{A} \in \mathcal{DFA}_{n}(\Sigma)~~~  |\mathcal{L_{\mathbb{P}}(\mathcal{A})} - \mathcal{L_{S}(\mathcal{A})}| \geq \epsilon(m,n, w(n)\delta) \}$ \\
Note that $E_{n}$ is nothing but the complement event given in Theorem 4.2.6, when $\mathcal{H} = \mathcal{DFA}_{n}(\Sigma)$.  Fix $\delta \in (0,1)$.We have, by means of the union bound: \\ 
\begin{equation*}
         \begin{aligned}
           \mathbb{P}(E) &\leq \sum\limits_{n \geq 1} \mathbb{P}(E_{n}) \\ 
           & \leq \sum\limits_{n \geq 1} 2^{-n} \delta  \\
           & \leq \delta \\
         \end{aligned}
\end{equation*}
 where the second inequality is obtained by Lemma 5.1 by setting $\mathcal{H} = \mathcal{DFA}_{n}(\Sigma)$, and using the VC-dimension in Lemma 4.2.5.
\end{proof} 
We wrap up this section by giving our main theorem, stated earlier informally, with the obtained bound for the case of SRM strategy:
\begin{theorem}
 Under assumption 1, for any $\epsilon,\delta \in (0,1)$, the learning setup will require $O(\frac{|\Sigma|\log(\frac{1}{\epsilon}) + \log(\frac{1}{\delta})}{\epsilon^{3}})$ sampling and membership queries to reach an approximation error smaller than $\epsilon$ with probability at least $1-\delta$.   
\end{theorem}
\begin{proof} 
 Let $C$ be the constant in theorem 4.2.6. Fix any $\epsilon \in (0,1)$. Let $\mathcal{A}_{t}$ be the DFA returned by the learning protocol, and $\mathcal{A}_{min}$ be any automaton that belongs to the set $ A_{min}(\epsilon)$, and $|\mathcal{A}_{min}| \leq c \frac{1}{\epsilon}$.  The existence of such automaton is proven in proposition 4.2.3. By theorem 4.2.6, we have  with probability at least $1-\delta$, the following two events occur jointly
  \begin{equation*}
  \mathcal{L}_{\mathbb{P}}(\mathcal{A}_{t}) \leq \mathcal{L}_{S}(\mathcal{A}_{min}) + C\sqrt{\frac{|\Sigma||\mathcal{A}_{min}|(\log(|\mathcal{A}_{min}|) + 1)  +  \log(\frac{1}{\delta})}{m}}
 \end{equation*}
 \begin{equation*}
     \mathcal{L}_{S}(\mathcal{A}_{min}) \leq \mathcal{L}_{\mathbb{P}}(\mathcal{A}_{min}) + C\sqrt{\frac{|\Sigma||\mathcal{A}_{min}|(\log(|\mathcal{A}_{min}|) + 1)  +  \log(\frac{1}{\delta})}{m}}
 \end{equation*}
Combining these two inequalities, we have 
\begin{equation*}
     \begin{aligned}
     \mathcal{L}_{\mathbb{P}}(\mathcal{A}_{t}) &\leq \mathcal{L}_{\mathbb{P}}(\mathcal{A}_{min}) + 2C\sqrt{\frac{|\Sigma||\mathcal{A}_{min}|(\log(|\mathcal{A}_{min}|) + 1)  +  \log(\frac{1}{\delta})}{m}} \\ 
     & \leq \inf\limits_{\mathcal{A} \in \mathcal{DFA}(\Sigma)} \mathcal{L}_{\mathbb{P}}(\mathcal{A}_{t}) + \epsilon + 2C \sqrt{\frac{c|\Sigma|\frac{1}{\epsilon}(\log(\frac{c}{\epsilon}) + 1)  +  \log(\frac{1}{\delta})}{m}}
     \end{aligned}
\end{equation*}
We pick $m$ such that $2C \sqrt{\frac{c|\Sigma|\frac{1}{\epsilon}(\log(\frac{c}{\epsilon}) + 1)  +  \log(\frac{1}{\delta})}{m}} \leq \epsilon$. Neglecting constant terms, we obtain the result of the theorem. 

\end{proof}

\subsection{Extracting FSMs with the most probable string oracle}
 In this section, we analyze the learning strategy where a learner has access to an oracle that provides the highest probable strings, and the membership query oracle. This arises as an intuitive learning strategy and represents the most favorable learning conditions under which a learner can learner. The learning protocol is simple and runs as follow:
 
 \begin{enumerate}
     \item The highest probable string oracle provides a sample S of distinct elements of size $n$ such that $\forall w \in \Sigma^{*} \setminus S,~ \forall s \in S:~ \mathbb{P}(s) \geq \mathbb{P}(w)$. 
     \item The learner makes $|S|$ membership queries to label elements of $S$,
     \item The learner returns a consistent automaton $\mathcal{A}$ with $S$,
 \end{enumerate}
 
 Note that no randomness is involved under this protocol, which makes learning guarantees in this setting more powerful than PAC-like guarantees.  \\ 
 Intuitively, if the probability distribution holds a high entropy, such protocol would require asymptotically an excessively high number of queries to obtain a good approximation of the target function. i.e. convergence properties of such protocol couldn't be distribution-free as for the SRM case examined in the previous case. So, a question arises on the class of distributions for which such protocol would converge. In the following, we shall prove that the class of sub-exponential distributions requires a number of queries that scales linearly with the desired error to converge under the presented protocol, when the size of the alphabet is small. This class of distributions emcompasses all distributions of finite support, for instance. Also, the class of PDFA distributions belongs to this class(\cite{balle13}). From a practical point of view, this makes PDFAs good candidates to serve as an oracle under this learning setting. Let's first define the class of sub-exponential distributions: 
\begin{definition}
  A probability distribution $\mathbb{P}$ on $\Sigma^{*}$ is sub-exponential if there exists $c > 0$ such that: $$\forall t \in \mathbb{N}:~~\mathbb{P}(|x| \geq t) \leq \exp(-ct)$$
\end{definition}
We present in the following the number of oracle queries required for convergence. 
\begin{theorem}
  If the langage model of reference $\mathbb{P}$ is sub-exponential, then learning with the highest probable string oracle (resp. the membership oracle) would require $O(|\Sigma|^{\log{\frac{1}{\epsilon}}})$ to attain: $\mathcal{L}_{\mathbb{P}}(\mathcal{A}) \leq \epsilon$.
\end{theorem}
\begin{proof}
  Let $\mathbb{P}$ be a sub-exponential distribution with parameter $c$. Fix any $\epsilon \in (0,1)$. Define $t = \frac{1}{c} \log\frac{1}{\epsilon}$. We have: 
  \begin{equation*}
       \mathbb{P}(|x| \geq t) \leq \epsilon 
  \end{equation*}
 According to the inequality above, If the learner queries $n = |\Sigma^{< t}| = O(|\Sigma|^{t})$ strings from the highest probable string oracle, and the oracle returns a set $S$, then necessarily, by the property of the oracle, we'll have: $\mathbb{P}(S) \geq \mathbb{P}(x < t)$. Hence, by the inequality above, we'll have: $\mathbb{P}(S) \geq 1 - \epsilon$. Hence, any consistent automaton $\mathcal{A}$ with this set will have: $\mathcal{L}_{\mathbb{P}}(\mathcal{A}) \leq \epsilon$. 
\end{proof}

\textbf{Interpretation of the number of queries bound.} Our obtained bounds suggest that the strategy of highest probable strings would be catastrophic in the context of languages over an alphabet with large size, such as the NLP case. If the alphabet size is small (i.e. $|\Sigma| = O(1)$), then the bound is reduced to be linear $O(\epsilon)$. Bounds obtained by the SRM strategy depends only with $O(\sqrt{\Sigma})$ rate on the alphabet. However, It's necessary to mention that this bound is very conservative, and is derived for the worst-case scenario, where all words for a given length has the same probability of occurence. This scenario is obviously unrealistic in practice. If one assumes smaller entropy on the set of words of a given length, one could derive tighter bounds.   \\ 

\subsection{Conclusion}
In this part of the thesis, we discussed learnability aspects of RNNs by deterministic finite state automata. We proposed an extension of the active learning framework that best fits our target objective: Efficient extraction of automata that holds as good generalization capabilities as the original target RNN. We motivated the necessity of the presence of a language model of reference that mediate between the RNN and its automaton approximation. We provided a theoretical analysis of two learning scenarios derived from our framework, and provable guarantees of their convergence were given. 

\section*{Conclusion and Perspectives}
\label{chap:discussion}
The work presented in this Master's thesis presents a first attempt to formalize notions related to learnability, computability and extractability of finite state machines from recurrent neural networks. The main objective was to explore, from a theoretical viewpoint, perspectives of connecting the symbolic world with its connectionist counterpart. From computability standpoint, we explored the class of complexity problems to which belong the equivalence and distance problem between RNNs trained as language models, and different types of weighted automata. We proved that deciding equivalence, computing distance or even computing an approximation of the distance between different types of weighted automata and RNN-LMs with $ReLu$ as an activation function are computationally hard problems. Afterwards, we discussed FSM extraction algorithms proposed in the literature, and conducted experiments to compare their respective performances. In light of our experiments, we revisited the longstanding \emph{Clustering hypothesis} of RNNs trained to recognize regular languages, and we give insights borrowed from the generalization theory of Deep Learning about this hypothesis. Our proposed conjecture goes to the direction that the \emph{clustering effect} is not a restricted property of RNNs well-trained to recognize regular languages, but rather an immediate effect to the smoothness property of all DL models that boast good generalization capabilities. Finally, we proposed an extension of the active learning framework tailored to fit the problem of FSM extraction from RNNs. We analyzed rigorously two learning scenarios of this framework, and gave sample complexity bounds to prove their PAC convergence. In our theoretical analysis, we gave a mild sufficient condition on the probability distribution and the target function which converts a non-uniformly PAC-learnable concept class into a uniform one. \\ 

$~~~$Our work was an opportunity to open new perspectives. We would like to cite some of them: 
\begin{itemize}[leftmargin=0cm]
    \item \textbf{Exploring connections between the ease of the \emph{"cluster-ability"} of the RNN hidden state space and the generalization capability of the network:} Our experiments point to the direction that there exists a close connection between the performance of algorithms based on clustering the hidden state space to extract automata and the generalization of the network. This observation suggests that DL models that boast good generalization performance are easier to compress into a Finite State Automaton. If this fact holds, It would mean that FSMs are good candidates to explain the capacity control of Deep Learning models during training, a fact that has important implications into the generalization theory of Deep Learning.
    \item \textbf{Analyzing other learning scenarios in the framework of learning RNNs with a reference language model oracle:} More can be done under the proposed framework in this thesis, such as analyzing incremental strategies of learning RNNs by Finite State Automata. An interesting scenario would be to use a probability weight oracle, and see at which extent a learning algorithm could learn both distribution of the reference language model and the target RNN, in the hope to minimize the number of queries to the probability weight oracle, as the learning algorithm would build a good approximation of the target distribution to use it on its own. 
    \item \textbf{Learning RNNs with high-margin outputs:} One of the major points that was dismissed in FSM extraction algorithms proposed in the litterature so far is that they consider the RNN to approximate as a binary classifier. However, RNN outputs usually results after the application of a sigmoid/softmax layer which delivers a real number (usually between $0$ and $1$. Thresholding the output to reduce it to a binary output is losing important information about the confidence the RNN has on predicting the output of a given input. The perspective of integrating this valuable information into FSM extraction algorithms needs to be explored. 
    \item \textbf{Extraction of transducers from Seq2Seq models.} Most of works proposed in the literature about the \emph{"extractability"} of rule-based systems from DL models focused mainly on classifiers. However, an important application in which RNNs excel is on Seq2Seq tasks. An analog of Seq2Seq computational models in the family of automata are transducers. An interesting perspective that is worth to be studied in the future is the extractability and learnability of transducers from RNNs.   
\end{itemize}

\printbibliography[heading=bibintoc]

\appendix 
\section{Report of experimental results}

\Large{\textbf{1. Running Time}} \\

\begin{table}[ht]
\begin{center}
\begin{tabular}{ccccc}
     \textbf{RNN Class} & \textbf{State size} & \textbf{G1} & \textbf{G4} & \textbf{G7}  \\
    \hline
    \multirow{0}{*}{Sigmoid order 1}& 50 & $[8.4,2.8,2.5]$ & $[13.4,3.6,2.8]$ & $[17.8,3.2,2.8]$  \\
    & 100 & $[16.4,3.6,2.8]$ & $[19.6,5.4,4.3]$ & $[22.1,5.2,4.5]$ \\
    & 150 & $[25.3,7.2,4.6]$ & $[29.5,7.9,6.8]$ & $[32.2,7.7,6.9]$\\
    \hline
    \multirow{0}{*}{\emph{tanh} order 1} & 50 & $[7.9,2.5,1.9]$ & $[8.2,3.4,2.5]$ & $[12.5,3.6,2.8]$\\
    & 100 & $[14.3,3.8,2.6]$ & $[16.2,4.9,4.3]$ & $[19.4,4.6,4.1]$\\
    & 150 & $[24.4,5.8,4.7]$ & $[28.4,7.2,6.4]$ & $[34.7,7.4,6.5]$\\
    \hline 
    \multirow{0}{*}{Sigmoid order 2} & 50 & $[6.9,1.4,1.2]$ & $[11.5,1.6,1.4]$ & $[15.4,1.3,1.3]$ \\
    & 100 & $[10.1,2.5,2.4]$& $[16.4,2.7,2.3]$ & $[22.4,2.6,2.2]$ \\ & 150 & $[14.6,4.4,3.7]$ & $[19.4,5.3,4.7]$ &  $[26.0,5.1,4.4]$\\
    \hline 
    \multirow{0}{*}{\emph{tanh} order 2} & 50 & $[7.2,1.5,1.1]$ & $[9.4,2.9,1.8]$ & $[6.4,2.8,2.0]$ \\
    & 100 & $[13.4,3.6,2.8]$ & $[22.2,4.2,3.1]$ & $[6.4,4.2,3.3]$ \\ 
    & 150 & $[23.8,5.6,4.3]$ & $[29.0.4,6.8,5.3]$ & $[35.7,6.5,5.1]$\\
    \hline 
    \multirow{0}{*}{LSTM} & 50 & $[24.2,2.2,1.8]$ & $[32.4,4.6,3.1]$ & $[44.4,3.2,2.8]$ \\
    & 100 & $[-,4.2,2.8]$ & $[-,7.2,5.9]$ & $[-,7.5,5.7]$\\
    & 150 & $[-,7.1,4.8]$ & $[-,7.6,5.2]$ & $[-,7.4,5.1]$ \\
    \hline 
    \multirow{0}{*}{GRU} & 50 & $[26.1,3.6,2.8]$ & $[31.8,3.6,2.8]$ & $[42.1,3.6,2.8]$ \\
    & 100 & $[-,4.6,3.2]$ & $[-,6.9,6.2]$ & $[-,7.2,6.4]$\\
    & 150 & $[-,7.9,5.2]$ & $[-,8.8,5.6]$ & $[-,7.3,5.4]$\\
\end{tabular}
\end{center}
\caption{Average running time in seconds of different algorithms for different class of RNN architectures: In each list $[.,.,.]$, the first element represents performances of the quantization algorithm, the second represents the clustering algorithm with K-means($K=15$), and the last is Weiss et al.\cite{Weiss18b} adaptation of the $L*$. Cells with $-$ makes reference to experiments that fail to achieve termination after a crash memory. $Gx$ in the table is abbreviation of Tomita Grammar number x}
\end{table}

\newpage

\Large{\textbf{2. Average Success Rate}}\\

\begin{table}[ht]
\begin{center}
\begin{tabular}{ccccc}
     \textbf{RNN Class} & \textbf{State size} & \textbf{G1} & \textbf{G4} & \textbf{G7}  \\
    \hline
    \multirow{0}{*}{Sigmoid order 1}& 50 & $[0,100,100]$ & $[0,80.0,100]$ & $[0,100,100]$  \\
    & 100 & $[0,80,100]$ & $[0,80,80]$ & $[0,60,100]$ \\
    & 150 & $[0,60,80]$ & $[0,40,80]$ & $[0,60,100]$\\
    \hline
    \multirow{0}{*}{\emph{tanh} order 1} & 50 & $[20,100,100]$ & $[0,80,80]$ & $[0,80,100]$\\
    & 100 & $[0,60,100]$ & $[0,80,80]$ & $[0,80,100]$\\
    & 150 & $[0,20,80]$ & $[0,20,60]$ & $[0,60,100]$\\
    \hline 
    \multirow{0}{*}{Sigmoid order 2} & 50 & $[0,100,100]$ & $[0,100,100.0]$ & $[0,100.0,100.0]$ \\
    & 100 & $[0,60,80]$& $[0,80.0,100.0]$ & $[0,100,100]$ \\ 
    & 150 & $[0,40,80]$ & $[0,20,60]$ &  $[0,60,80]$\\
    \hline 
    \multirow{0}{*}{\emph{tanh} order 2} & 50 & $[20,100,100]$ & $[0,80,100]$ & $[0,80,100]$ \\
    & 100 & $[0,60,100]$ & $[0,40,80]$ & $[0,80,100]$ \\ 
    & 150 & $[0,60,80]$ & $[0,20,80]$ & $[0,80,100]$\\
    \hline 
    \multirow{0}{*}{LSTM} & 50 & $[0,80,100]$ & $[0,80,100]$ & $[0,80,100]$ \\
    & 100 & $[-,40,100]$ & $[-,40,80]$ & $[-,80,100]$\\
    & 150 & $[-,20,80]$ & $[-,20,80]$ & $[-,80,100]$ \\
    \hline 
    \multirow{0}{*}{GRU} & 50 & $[0,80,80]$ & $[0,80,100]$ & $[0,100.0,100.0]$ \\
    & 100 & $[-,60,80]$ & $[-,40,80]$ & $[-,80,100]$\\
    & 150 & $[-,20,80]$ & $[-,60,80]$ & $[-,60,100]$\\
\end{tabular}
\end{center}
\caption{Average success rate in $\%$ of different algorithms for different class of RNN architectures for 5 different RNN training runs. In each list $[.,.,.]$, the first element represents performances of the quantization algorithm, the second represents the clustering algorithm with K-means($K=15$), and the last is Weis et al.\cite{Weiss18a} adaptation of the $L*$. Cells with $-$ makes reference to experiments that fail to achieve termination after a crash memory. $Gx$ in the table is abbreviation of Tomita Grammar number x}
\end{table}

\end{document}